\PassOptionsToPackage{table}{xcolor}
\documentclass[10pt,journal]{IEEEtran}
\usepackage[T1,OT1]{fontenc}

\usepackage{amsmath,amsfonts}
\usepackage{amssymb}
\usepackage{mathtools}
\usepackage{amsthm}

\usepackage{float}
\usepackage{algorithm}
\usepackage{algorithmic}
\usepackage{booktabs}
\usepackage{graphicx}
\usepackage[table]{xcolor}

\usepackage{array}
\usepackage{tabularx}
\usepackage{multirow}

\usepackage{colortbl}
\usepackage{arydshln}

\usepackage[caption=false,font=normalsize,labelfont=sf,textfont=sf]{subfig}
\usepackage{textcomp}
\usepackage{stfloats}
\usepackage{url}
\usepackage{verbatim}
\usepackage{microtype}
\usepackage{bbm}
\usepackage{cite}

\definecolor{letheblue}{RGB}{218,232,252}

\usepackage{hyperref}
\usepackage[capitalize,noabbrev]{cleveref}

\graphicspath{{figs/}}
\newcommand{\bw}{\boldsymbol{w}}

\newcommand{\method}{\textnormal{\textsc{Lethe}}}

\theoremstyle{plain}
\newtheorem{theorem}{Theorem}[section]
\newtheorem{proposition}{Proposition}[section]

\newtheorem{corollary}{Corollary}[section]
\theoremstyle{definition}

\theoremstyle{remark}

\begin{document}

\title{Securing the Right to be
Forgotten: Towards Persistent Knowledge Erasure in
Federated Unlearning via Principled Dual-Stream Update}

\author{Wentai~Wu,~\IEEEmembership{Member,~IEEE,}
        Hanwei~Tan,
        Yijun~Quan,
        Haixia Peng,
        Ligang~He,~\IEEEmembership{Member,~IEEE,}
        Bin Yang,
        C.~L.~Philip Chen,~\IEEEmembership{Fellow,~IEEE}

\thanks{W. Wu and H. Tan are with the Department of Computer Science, College of Information Science and Technology, Jinan University, Guangzhou 510632, China. e-mails: wentaiwu@jnu.edu.cn, tanhanwei@stu2025.jnu.edu.cn.}%
\thanks{Y. Quan is with WMG, University of Warwick, Coventry CV4 7AL, United Kingdom. e-mail: Yijun.quan@warwick.ac.uk.}%
\thanks{H. Peng is with the School of Information and Communications Engineering at Xi’an Jiaotong University, China. e-mail: haixia.peng@xjtu.edu.cn.}%
\thanks{L. He is with the Department of Computer Science, University of Warwick, Coventry CV4 7AL, United Kingdom. e-mail: ligang.he@warwick.ac.uk.}%
\thanks{B. Yang is with the School of Data Science and Engineering, East China Normal University, Shanghai 200050, China. e-mail: byang@cs.aau.dk.}%
\thanks{C.L.P. Chen is with the School of Computer Science and Engineering, South China University of Technology. e-mail: philipchen@scut.edu.cn}%
\thanks{Corresponding author: Hanwei Tan.}
} 


\markboth{Preprint}%
{Wu \MakeLowercase{\textit{et al.}}: Securing the Right to be Forgotten: Towards Persistent Knowledge Erasure in Federated Unlearning via Principled Dual-Stream Update}


\maketitle

\begin{abstract}
Federated unlearning (FU) aims to erase knowledge from a global model. Most FU evaluations
 emphasize immediate forgetting, paying limited attention to its persistence under subsequent
 training over the remaining clients. In this situation, \emph{knowledge resurfacing} poses a critical failure mode:
 continued training on retained data alone can reactivate unlearned knowledge in a few rounds.
 Empirically, we demonstrate that many 
 state-of-the-art FU methods are prone to knowledge resurfacing, which implies unreliable 
 forgetting and raises trust issues. To this end, we propose \textsc{Lethe}, a novel 
 unlearning method for persistent knowledge erasure under federated settings. During unlearning, 
  \textsc{Lethe}  operates on a forget stream from the unlearning client and a retain 
  stream from the retained clients. It uses dual-stream update rectification
  during unlearning to suppress subsequent target refitting. Our theoretical analysis bounds the
  forget-set loss recovery over a finite training horizon. We also present recovery-alignment margin, parameter rollback, and path-length condition to interpret knowledge resurfacing and why \textsc{Lethe} is effective against it. Extensive experiments across diverse models, datasets, and unlearning settings demonstrate that \textsc{Lethe} is 
     effective for unlearning on both CV and NLP tasks,
 consistently suppressing undesired knowledge retention even after an extremely
     long horizon of continued federated training.
\end{abstract}


\section{Introduction}
\IEEEPARstart{F}{ederated} learning (FL) enables distributed clients to collaboratively train models without sharing their raw data~\cite{mcmahan2017communication}. Although FL avoids the direct exposure of user data, the resulting global model may still learn and retain user-specific knowledge from private training data~\cite{liu2023surveyfu,liu2023surveyfu1}. Consequently, privacy risks may persist even after a user opts out of the federation. To comply with privacy regulations such as the ``right to be forgotten'' under the GDPR~\cite{garg2020formalizing,liu2023surveyfu1}, users should be able to request the removal of knowledge associated with their data from the trained model. Federated unlearning (FU) has therefore emerged as a promising approach to addressing this requirement by removing the influence of specific data from the global model, with the goal of making the resulting model behave as if those data had never been involved in federated training~\cite{liu2023surveyfu,liu2023surveyfu1}.

In real-world federated deployment, \textbf{unlearning requests may arrive in the middle of a long-running training process}, meaning that federated training may continue after the unlearning is fulfilled. This adds a persistent requirement to FU beyond effective knowledge erasure. If continued federated training on the retained data reactivates the removed knowledge, the trustworthiness of unlearning and the security of forgotten data can be substantially compromised. Therefore, a key challenge for FU is maintaining the effect of unlearning during continued training.

\begin{figure}[t]
    \centering
    \includegraphics[width=0.98\columnwidth]{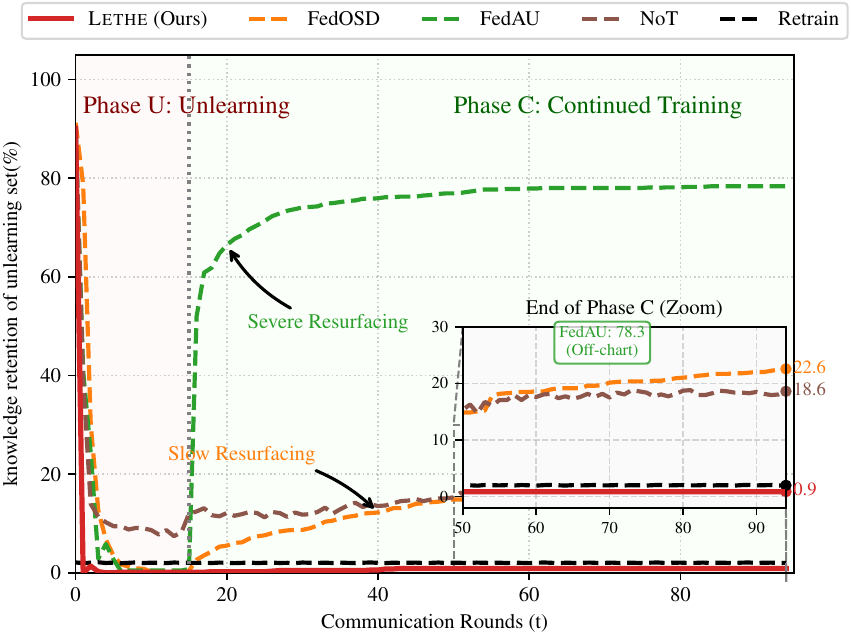}
    \caption{
    Knowledge retention on the unlearning set during unlearning (phase U) and continued training (phase C). Lower values indicate stronger removal of target knowledge. Existing methods quickly reduce knowledge retention during unlearning, but the model relearns what is supposed to be deleted during continued training, i.e., knowledge resurfacing. In contrast, our \textsc{Lethe} keeps unlearned knowledge deactivated throughout continued training.
    }
    \label{fig:top_right_curve}
\end{figure}

Retraining on the retained data is regarded as the gold standard for unlearning, as the resulting model has never been trained on the target data. However, such retraining is often prohibitive in terms of computation and communication at scale. As solutions, existing unlearning strategies can be broadly categorized into exact and approximate unlearning~\cite{thudi2022necessity,suriyakumar2022algorithms,liu2023surveyfu1}.

Exact unlearning aims to obtain a model whose behavior is indistinguishable from retraining after removing the target data~\cite{thudi2022necessity,suriyakumar2022algorithms}. Such methods typically accelerate retraining by storing historical gradients or using inverse-Hessian-based techniques, but often incur substantial storage and computational overhead. In contrast, approximate unlearning aims to reduce target knowledge retention to a level comparable to retraining~\cite{suriyakumar2022algorithms,thudi2022necessity}, while lowering time and resource costs. It typically requires no additional historical storage and can complete unlearning within only a few communication rounds.

This distinction also suggests that approximate unlearning does not necessarily imply complete removal of the target knowledge. Although an unlearned model may exhibit a substantially reduced influence of the target data immediately after unlearning, it may still retain a latent tendency to reactivate the unlearned knowledge during continued federated training~\cite{lynch2024eight}. Prior studies~\cite{lynch2024eight,wang2025invariance} have shown that unlearned models may retain residual knowledge associated with the target data, remain vulnerable to poisoning effects, or even recover the supposedly removed behavior after only limited retraining. As shown in Figure~\ref{fig:top_right_curve}, although approximate unlearning methods can rapidly reduce the retention of target knowledge (Phase U), \textbf{continued federated training on the retained data alone can reactivate the removed knowledge (Phase C)}, as reflected by the recovery of model performance on the unlearned patterns. In this work, we refer to this recovery induced by retained-data training as \textit{knowledge resurfacing}~\cite{lynch2024eight}. By contrast, a model retrained from scratch solely on the retained data does not exhibit this issue. This observation suggests that a central challenge for approximate unlearning is not only to remove target knowledge during unlearning, but also to maintain the effect of unlearning during continued federated training and prevent knowledge resurfacing.

\begin{figure*}[t]
  \centering
  \includegraphics[width=\textwidth]{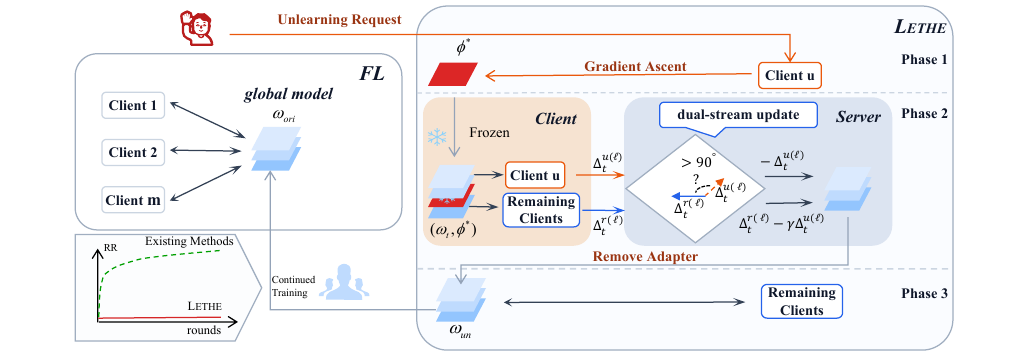}
    \caption{\textbf{\textsc{Lethe} for persistent federated unlearning.}
    \textsc{Lethe} acts in the unlearning phase via dual-stream update rectification. 
    The forget stream captures the target-knowledge recovery direction from the unlearning data, while the retain stream preserves utility from retained data. 
    By de-correlating the two streams via update rectification, \textsc{Lethe} suppresses potential knowledge resurfacing during subsequent federated training.}
  \label{fig:framework}
\end{figure*}

Moreover, although LLM unlearning has recently attracted substantial attention in centralized settings, federated LLM unlearning remains comparatively under-explored~\cite{zhang2026oblivionis}. This gap is important because federated LLM training is typically conducted via parameter-efficient fine-tuning, such as LoRA~\cite{hu2022lora}, rather than full-parameter optimization, due to the high communication and local training costs of large models~\cite{zhang2026oblivionis}. This difference in training paradigm calls for FU methods and evaluation protocols that can be adapted to LLM-based NLP unlearning settings.

To address these challenges, we propose \textsc{Lethe}\footnote{Named after ``the river of oblivion''
 in Greek mythology.}, a federated unlearning method designed to mitigate knowledge resurfacing when
  federated training continues on retained data. As shown in Fig.~\ref{fig:framework}, \textsc{Lethe} 
  constructs a \emph{forget stream} from the unlearning data and a \emph{retain stream} from the 
  retained data. \textbf{During unlearning, \textsc{Lethe} adaptively rectifies the directions of 
  the two streams to suppress update that may facilitate the subsequent recovery 
  of removed knowledge while preserving retained utility.} This rectification is applied only 
  during the unlearning stage and reshapes the model state returned for deployment. Subsequent 
  training proceeds with ordinary retained-data FL, without additional rectification, access to 
  the unlearning objective, or auxiliary regularization. Our experiments show that \textsc{Lethe} 
  better preserves the effect of unlearning under continued training. The same framework 
  supports sample-, class-, and client-level unlearning across both CV and LLM tasks, with only
   the construction of the two streams adapted to each setting.

Theoretically, we examine feature representations, rollback tendency, and linear mode connectivity in Sec.~\ref{sec:analysis} to diagnose residual target knowledge and its recovery during continued training. Section~\ref{sec:theory-analysis} then connects the rectification identity, local alignment stability, and a finite path budget to a conditional upper bound on forget-loss recovery over a prescribed continued-training horizon.

Our key contributions are summarized as follows:

\begin{itemize}
\item \textbf{Characterization of knowledge resurfacing:} 
We formulate continued training as a persistence evaluation for FU and introduce \emph{Knowledge Retention (KR)} and \emph{Resurfacing Rate (RR)} to measure immediate unlearning effectiveness and subsequent recovery of the target knowledge. 

\item \textbf{Methodology and analysis:} We propose \textsc{Lethe}, a federated unlearning method based on a three-stage unlearning design. 
The core stage adopts dual-stream update rectification to reduce the target-refitting component of the applied update.
We characterize this update-level effect and derive a finite-horizon forget-loss-recovery bound.

\item \textbf{Evaluation:} \textsc{Lethe} consistently achieves near-zero resurfacing (RR $<$ 1\% in most cases) across various CV and NLP tasks, outperforming existing baselines while maintaining competitive utility.
\end{itemize}

\section{Related Work}

\paragraph{Federated unlearning}
Approaches to federated unlearning are generally categorized into exact and approximate unlearning~\cite{thudi2022necessity,suriyakumar2022algorithms,liu2023surveyfu1}.

Exact unlearning seeks to obtain an unlearned model whose behavior is indistinguishable from that of a model retrained on the retained data. FedEraser~\cite{federaser} accelerates retraining by reusing historical client updates, while FATS~\cite{tao2024communication} enables efficient exact unlearning through total-variation-stable training. Quan et al.~\cite{quan2026exact} further support exact federated continual unlearning for ridge heads on frozen foundation models via additive sufficient statistics.

Approximate unlearning relaxes this requirement and instead aims to approximate the behavior of the retrained model without requiring parameter-level equivalence. FedAU~\cite{fedau} introduces an auxiliary unlearning module during training and adopts a linear neutralization operation during the unlearning stage to offset the influence of the target data on the model. NoT~\cite{fednot} performs unlearning by applying weight negation to model parameters followed by subsequent fine-tuning. FedOSD~\cite{pan2025fedosd} constructs an orthogonal steepest descent direction that minimizes conflicts with the gradients of other clients for client-level unlearning. FedU~\cite{wang2024fedu} mainly targets the unlearning of user-specified samples, allowing the user who submits the unlearning request to locally estimate the influence of the removed samples on the global model and subtract it from the model. 

\paragraph{LLM unlearning}

LLM unlearning aims to remove specific knowledge from pre-trained language models~\cite{wang2025invariance,zhang2026oblivionis}. Gradient ascent (GA)~\cite{grandent-ac} increases the loss on forget data to suppress target generation. NPO~\cite{NPO} treats forget responses as negative preferences to alleviate the instability of GA, while SimNPO~\cite{simNPO} further removes the reference model with a simpler reference-free objective.

Overall, existing LLM unlearning methods are largely developed and evaluated in centralized settings on NLP-oriented generative tasks. As a result, they are not readily transferable to image-based or federated unlearning scenarios, where the underlying model architectures, optimization pipelines, and evaluation protocols differ substantially.

\begin{table}[t]
    \centering
    \caption{Positioning of representative unlearning methods. ``Fed.'' denotes original design
    for FU; ``Persist.'' reports behavior under our retained-data continued-training protocol.}
    \label{tab:baseline_positioning}
    \scriptsize
    \setlength{\tabcolsep}{3pt}
    \renewcommand{\arraystretch}{1.05}
    \resizebox{\columnwidth}{!}{
    \begin{tabular}{llccccc}
        \toprule
        \textbf{Domain} 
        & \textbf{Method} 
        & \textbf{Fed.} 
        & \textbf{Sample} 
        & \textbf{Class} 
        & \textbf{Client} 
        & \textbf{Persist.} \\
        \midrule

        \multirow{5}{*}{CV}
        
        & FedEraser ~\cite{federaser}
        & \checkmark 
        & $\times$ 
        & $\times$ 
        & \checkmark 
        & \checkmark \\

        & FedAU ~\cite{fedau}
        & \checkmark 
        & \checkmark 
        & \checkmark 
        & \checkmark 
        & $\triangle$ \\

        & FedOSD ~\cite{pan2025fedosd}
        & \checkmark 
        & $\times$ 
        & $\times$ 
        & \checkmark 
        & $\times$ \\

        & NoT ~\cite{fednot}
        & \checkmark 
        & \checkmark 
        & \checkmark 
        & \checkmark 
        & $\triangle$ \\

        & \cellcolor{letheblue}\textsc{Lethe}
        & \cellcolor{letheblue}\checkmark 
        & \cellcolor{letheblue}\checkmark 
        & \cellcolor{letheblue}\checkmark 
        & \cellcolor{letheblue}\checkmark 
        & \cellcolor{letheblue}\checkmark \\
        \midrule

        \multirow{3}{*}{NLP}

        & NPO~\cite{NPO}
        & $\times$ 
        & \checkmark 
        & - 
        & $\times$ 
        & $\times$ \\

        & simNPO ~\cite{simNPO}
        & $\times$ 
        & \checkmark 
        & - 
        & $\times$ 
        & $\times$ \\

        & \cellcolor{letheblue}\textsc{Lethe}
        & \cellcolor{letheblue}\checkmark 
        & \cellcolor{letheblue}\checkmark 
        & \cellcolor{letheblue}-
        & \cellcolor{letheblue}\checkmark 
        & \cellcolor{letheblue}\checkmark \\
        \bottomrule
    \end{tabular}
    }
    \vspace{0.5mm}
    
    \footnotesize{
    Outside ``Persist.'': \(\checkmark\), supported; \(\times\), unsupported.
    Under ``Persist.'': \(\checkmark\), persistent; \(\triangle\), setting-dependent; \(\times\), substantial recovery.
    }
\end{table}

\paragraph{Knowledge resurfacing}
Robust unlearning studies increasingly distinguish the immediate effect of unlearning from model behavior under later interventions. In centralized LLM settings, Lynch et al.~\cite{lynch2024eight} develop evaluations for latent or extractable knowledge, while \L{}ucki et al.~\cite{lucki2024adversarial} show that adversarial extraction and small-scale fine-tuning can recover suppressed capabilities. Wang et al.~\cite{wang2025invariance} study unrelated downstream fine-tuning and improve robustness with an invariance-based objective.

Recovery after unlearning also appears outside the LLM setting. Siddiqui et al.~\cite{siddiqui2025dormant} study centralized, example-level image classification and show that retain-only fine-tuning can restore forget-set accuracy; their analysis connects this recovery to $L_2$ distance and linear mode connectivity and motivates weight-space regularization. FedOSD~\cite{pan2025fedosd} addresses client-level federated unlearning, constructing a non-conflicting descent direction and coupling unlearning with a utility-recovery stage intended to preserve the achieved unlearning effect. These studies establish post-unlearning recovery in different threat models and optimization regimes.

In contrast to these studies, \textsc{Lethe} acts only during the unlearning stage, before the unlearned model enters subsequent federated training. Its dual-stream rectification shapes the unlearned model for subsequent retained-client training, which proceeds without stored update histories, continued access to the forget objective, or additional anti-resurfacing constraints.

We summarize the features of representative studies and compare them with our method in Table~\ref{tab:baseline_positioning}.

\section{Preliminaries and Problem Statement}
\label{sec:preliminaries}
We consider an FU scenario where a client requests unlearning with regard to the global model which has been trained through FL. The unlearning operation involves the whole federation and is followed by continued FL over the remaining clients. This setting is practical in production scenarios where unlearning requests arrive amid long-term federated training.


\subsection{Problem Formulation}
Let $\mathcal{D}_k$ be the local data of client $k$ in a federation $\mathcal{K}$, and $\mathcal{D}=\cup_{k\in \mathcal{K}}\mathcal{D}_k$.
Let $\boldsymbol{w}^{(\ell)}\in\mathbb{R}^{d}$ denote the parameter vector of layer $\ell$. Let $\mathcal{A}$ denote the federated learning algorithm and $\bw_0$ be an initialized model.
For example, at round $t$, FedAvg \cite{mcmahan2017communication} updates the global model by
\begin{equation}
\boldsymbol{w}_{t+1}^{(\ell)}=\boldsymbol{w}_{t}^{(\ell)} + \sum_{k\in\mathcal{K}_{r,t}} q_{t,k}\,\Delta \boldsymbol{w}_{t,k}^{(\ell)},
\label{eq:layer_agg_retrain}
\end{equation}
where $\mathcal{K}_{r,t}$ is the selected client set, $\Delta \boldsymbol{w}_{t,k}^{(\ell)} \triangleq \boldsymbol{w}_{t,k}^{(\ell)}-\boldsymbol{w}_{t}^{(\ell)}$
is the local update, $n_{t,k}^{r}$ is the number of retained samples used by client $k$, and $q_{t,k}=n_{t,k}^{r}/N_t^{r}$,
 where $N_t^{r}=\sum_{j\in\mathcal{K}_{r,t}}n_{t,j}^{r}$. 
After iterative training, the original global model (before any unlearning request) can be expressed as
$\bw_{\mathrm{ori}}=\mathcal{A} (\bw_0, \mathcal{D})$.

Let $\mathcal{D}_{u}\subset \mathcal{D}$ denote the designated unlearning set, and let $\mathcal{D}_{r}=\mathcal{D}\setminus \mathcal{D}_{u}$ denote the retained set.
The unlearning algorithm $\mathcal{U}$ updates $\bw_{\mathrm{ori}}$ so as to remove knowledge of $\mathcal{D}_{u}$.
Given the original global model $\bw_{\mathrm{ori}}$, the unlearning set $\mathcal{D}_{u}$, and the retained set $\mathcal{D}_{r}$, the unlearned model is obtained by
$\bw_{\mathrm{un}}=\mathcal{U}(\bw_{\mathrm{ori}}, \mathcal{D}_{u}, \mathcal{D}_{r}).$
We consider approximate behavioral unlearning, with the clean-retraining model
 $\bw_{\mathrm{re}}=\mathcal{A}(\bw_0,\mathcal{D}_{r})$ serving as the reference. A successful output should preserve 
 retained utility and match the target-set behavior of clean retraining within pre-specified tolerances $\epsilon_r,\epsilon_u\ge0$, i.e.,
\begin{equation}
\begin{aligned}
&\mathcal{L}(\bw_{\mathrm{un}};\mathcal{D}_{r})
\le \mathcal{L}(\bw_{\mathrm{re}};\mathcal{D}_{r})+\epsilon_r,\\
&\left|\mathrm{KR}(\bw_{\mathrm{un}},\mathcal{D}_{u})
-\mathrm{KR}(\bw_{\mathrm{re}},\mathcal{D}_{u})\right|
\le \epsilon_u.
\end{aligned}
\label{eq:unlearning_success}
\end{equation}
Here, $\epsilon_r$ bounds retained-utility degradation and $\epsilon_u$ bounds deviation from the 
clean-retraining forgetting reference. 
Since we consider multiple stages in the federation lifecycle, we summarize the notations 
of model states in Table~\ref{tab:model_states} for clarity.

Federated unlearning can be categorized from two orthogonal aspects: the task domain and the unlearning granularity. 
In this work, we consider both CV and NLP tasks and the granularities of unlearning as follows:

(1) \emph{Sample unlearning}: remove the knowledge associated with specified samples from a client.

(2) \emph{Client unlearning}: erase the knowledge contributed by all local data of a designated client.

(3) \emph{Class unlearning}: remove the knowledge associated with a target class, where the target data consist of all samples belonging to that class.

Sample- and client-level unlearning are applicable to both CV and NLP tasks, while class-level unlearning makes sense in classification tasks. 
In this paper, we evaluate all three granularities in CV tasks and evaluate sample-level knowledge unlearning in NLP tasks.

\begin{table}[t]
\centering
\caption{Notation of model states.}
\label{tab:model_states}
\renewcommand{\arraystretch}{1.15}
\begin{tabularx}{\columnwidth}{@{}>{\centering\arraybackslash}p{0.23\columnwidth} X@{}}
\toprule
\textbf{Notation} & \textbf{Description} \\
\midrule
$\boldsymbol{w}_{\mathrm{ori}}$ 
& Original global model trained on the full dataset $\mathcal{D}$ before unlearning. \\

$\boldsymbol{w}_{\mathrm{un}}$ 
& Model immediately after removing target knowledge from $\mathcal{D}_u$. \\

$\boldsymbol{w}_{\mathrm{re}}$ 
& Clean retraining baseline trained from scratch using only $\mathcal{D}_r$. \\
\bottomrule
\end{tabularx}
\end{table}

\subsection{Measures of Knowledge Resurfacing}
Because CV and NLP tasks use different evaluation protocols, we introduce \emph{knowledge retention} (KR) as a unified measure of unlearning effectiveness. KR quantifies how much knowledge associated with $\mathcal{D}_u$ remains in the model. A smaller KR value indicates a stronger effect of unlearning. During continued training, knowledge resurfacing is identified when target-set performance exceeds this reference level.

\textbf{For classification tasks}, KR is instantiated as the prediction accuracy on the unlearning set~\cite{Chundawat_Tarun_Mandal_Kankanhalli_2023}:
\begin{equation}
\mathrm{KR_{cls}}(\bw,\mathcal{D}_u)
=
\frac{1}{|\mathcal{D}_u|}
\sum_{(x,y)\in\mathcal{D}_u}
\mathbbm{1}
\big[
f_{\bw}(x)=y
\big].
\end{equation}

\textbf{For generative LLM tasks}, KR is defined as the average per-sample ROUGE-L F1 between the
 generated and reference answers~\cite{lin2004rouge}. For each sample $x_i$, let $\hat{y}_i(\bw)$ 
 denote the generated answer and $y_i$ its reference. After lowercasing and whitespace tokenization,
  let $F_i^{\mathrm{RL}}$ denote the ROUGE-L F1 between $\hat{y}_i(\bw)$ and $y_i$, 
  with $F_i^{\mathrm{RL}}=0$ if either token sequence is empty. Then,
\begin{equation}
\mathrm{KR}_{\mathrm{llm}}(\bw,\mathcal{D}_u)
=
\frac{1}{|\mathcal{D}_u|}
\sum_{i=1}^{|\mathcal{D}_u|}
F_i^{\mathrm{RL}}.
\label{eq:llm_kr}
\end{equation}


Let $T$ denote the final checkpoint of the pre-specified continued-training
schedule, let $\mathrm{KR}^{*}$ denote the target KR that the unlearned model
is expected to approach at this checkpoint, and let $\mathrm{KR}^{\mathrm{ori}}$ denote that for $\boldsymbol{w}_{\mathrm{ori}}$. In practice, we use the KR of the
clean-retraining model evaluated at the same checkpoint as $\mathrm{KR}^{*}$.
The \textbf{resurfacing rate} (\textbf{RR}) is defined as
\begin{equation}
\mathrm{RR} \triangleq
\frac{
\left[\mathrm{KR}_{T}
-\mathrm{KR}^{*}\right]_{+}
}{
\mathrm{KR}^{\mathrm{ori}}
-\mathrm{KR}^{*}
}
\times 100\%.
\label{eq:rr}
\end{equation}
Here, $[a]_{+}=\max(0,a)$. If $\mathrm{KR}^{\mathrm{ori}}\le
\mathrm{KR}^{*}$, RR is reported as NA because the
normalization gap is non-positive. Thus, RR measures endpoint recovery beyond
the target KR rather than the maximum value along the trajectory.


\section{Methodology}
\label{sec:method}
To achieve reliably persistent unlearning, \textsc{Lethe} follows an ``\textbf{Amplify--Rectify--Restore}'' three-stage pipeline that is basically model- and task-agnostic. At its core lies the dual-stream update rectification mechanism by which \textsc{Lethe} redirects the optimization into the direction that consistently induces forgetting while discouraging the retain-set update from re-activating forgotten knowledge, thereby mitigating knowledge resurfacing during subsequent training.

\subsection{Phase 1: Amplify with a temporary adapter}

Given the original global model $\boldsymbol{w}_{\mathrm{ori}}$, \textsc{Lethe} injects a temporary adapter $\boldsymbol{\phi}$ to amplify the unlearning signal. 
The temporary adapter is trained locally by the unlearning client using the unlearning data $\mathcal{D}_u$. 
During this process, the backbone parameters $\boldsymbol{w}_{\mathrm{ori}}$ are frozen, and only the temporary adapter is updated by gradient ascent:
\begin{equation}
\boldsymbol{\phi}^{s+1}
=
\boldsymbol{\phi}^{s}
+
\eta_{\phi}
\nabla_{\boldsymbol{\phi}}
\mathcal{L}(\boldsymbol{w}_{\mathrm{ori}}, \boldsymbol{\phi}^{s}; \mathcal{D}_u),
\label{eq:probe_train}
\end{equation}
where $\eta_{\phi}$ is the adapter learning rate. The adapter training stops after $S_{\phi}$ steps, and the resulting temporary adapter is denoted as
$\boldsymbol{\phi}^{*}=\boldsymbol{\phi}^{S_{\phi}}$. It is then sent to the server and broadcast across all the clients.

The adoption of temporary adapter is optional but practically useful for sample-level unlearning. \textsc{Lethe} remains effective in class- and client-level unlearning without it, and our NLP experiments omit it. Appendices S3-A, S3-B and S4-A report its parameterization, lifecycle, and ablation.

\subsection{Phase 2: Rectify via dual-stream update}

At each rectification round $t$, the server broadcasts only the current backbone $\boldsymbol{w}_t$.
Each participating client combines $\boldsymbol{w}_t$ with the received $\boldsymbol{\phi}^{*}$ and performs local training, where $\boldsymbol{\phi}^{*}$ remains fixed and only the backbone parameters are updated. 
This produces two update streams: the forget stream from $\mathcal{D}_u$ and the retain stream from retained clients.

\paragraph*{1) \textbf{Forget stream}}
The unlearning client computes the forget-stream update on $\mathcal{D}_u$:
\begin{equation}
\Delta_t^{u}
\triangleq
\mathrm{LocalTrain}(\boldsymbol{w}_t,\mathrm{frozen}\ \boldsymbol{\phi}^{*};\mathcal{D}_u)
-
\boldsymbol{w}_t .
\label{eq:forget_stream}
\end{equation}

This stream provides the update direction associated with the target knowledge to be removed.

\paragraph*{2) \textbf{Retain stream}}
In parallel, retained clients compute the retain-stream update on their local retained data. 
Let $\mathcal{B}_t \subseteq \mathcal{C}_r$ denote the sampled retained clients at round $t$. 
The server aggregates their backbone updates as
\begin{equation}
\Delta_t^{r}
\triangleq
\sum_{k\in\mathcal{B}_t}
p_k
\left(
\mathrm{LocalTrain}(\boldsymbol{w}_t,\mathrm{frozen}\ \boldsymbol{\phi}^{*};\mathcal{D}_k)
-
\boldsymbol{w}_t
\right),
\label{eq:retain_stream}
\end{equation}
where $p_k$ is the aggregation weight of client $k$. 
This stream provides the utility-preserving update direction.

\begin{algorithm}[tb]
   \caption{\textsc{Lethe}}
   \label{alg:Lethe}
\begin{algorithmic}[1]
   \STATE {\bfseries Input:} original global model $\boldsymbol{w}_{\mathrm{ori}}$; 
   unlearning data $\mathcal{D}_u$; retained clients $\mathcal{C}_r$; 
   penalty coefficient $\gamma$; adapter steps $S_{\phi}$; rectification rounds $T$; recovery rounds $R$.
   \STATE {\bfseries Output:} unlearned model $\boldsymbol{w}_{\mathrm{un}}$.
   \STATE $\boldsymbol{w}_0 \leftarrow \boldsymbol{w}_{\mathrm{ori}}$.

   \STATE \textit{// Phase 1: Amplify with a temporary adapter}
   \STATE Attach a lightweight temporary adapter $\boldsymbol{\phi}$ to $\boldsymbol{w}_0$.
   \STATE Freeze the backbone $\boldsymbol{w}_0$ and update only $\boldsymbol{\phi}$ by Eq.~\eqref{eq:probe_train}.
   \STATE Obtain the temporary adapter $\boldsymbol{\phi}^{*} \leftarrow \boldsymbol{\phi}^{S_{\phi}}$.
   \STATE Upload $\boldsymbol{\phi}^{*}$ to the server once and broadcast it once to the rest of clients.

   \STATE \textit{// Phase 2: Rectify via dual-stream update}
   \FOR{$t=0$ {\bfseries to} $T-1$}
      \STATE The server broadcasts the current backbone $\boldsymbol{w}_t$.
      \STATE Clients form $(\boldsymbol{w}_t,\boldsymbol{\phi}^{*})$ using the frozen $\boldsymbol{\phi}^{*}$.
      \STATE Construct the forget stream $\Delta_t^{u}$ from $\mathcal{D}_u$.
      \STATE Construct the retain stream $\Delta_t^{r}$ by aggregating retained-client updates from $\mathcal{C}_r$.
      \FOR{each layer $\ell$}
         \STATE Compute the stream alignment $s_t^{(\ell)}$ by Eq.~\eqref{eq:sim_def}.
         \STATE Obtain the rectified update $\widetilde{\Delta}_t^{(\ell)}$ by Eq.~\eqref{eq:layerwise_rectify}.
      \ENDFOR
      \STATE Update the backbone $\boldsymbol{w}_{t+1}$ by Eq.~\eqref{eq:update}.
   \ENDFOR

   \STATE \textit{// Phase 3: Restore retained utility}
   \STATE Discard the temporary adapter $\boldsymbol{\phi}^{*}$.
   \STATE Run $R$ rounds of FL over the retained clients $\mathcal{C}_r$ and obtain the final global model $\boldsymbol{w}_{\mathrm{un}}$.
   \STATE \textbf{return} $\boldsymbol{w}_{\mathrm{un}}$.

\end{algorithmic}
\end{algorithm}

\paragraph*{3) \textbf{Layer-wise rectification}}
After obtaining the forget stream $\Delta_t^{u}$ and the retain stream $\Delta_t^{r}$, the server performs dual-stream layer-wise rectification. 
For each layer $\ell$, the server first computes the alignment between the two streams:
\begin{equation}
s_t^{(\ell)}
\triangleq
\left\langle
\Delta_t^{r,(\ell)},
\Delta_t^{u,(\ell)}
\right\rangle .
\label{eq:sim_def}
\end{equation}

A positive $s_t^{(\ell)}$ indicates that the retain stream is positively aligned with the forget stream, revealing an adverse correlation that may facilitate subsequent target refitting. We use the following rule to rectify the actual update:
\begin{equation}
\widetilde{\Delta}_t^{(\ell)}
=
\begin{cases}
\Delta_t^{r,(\ell)} - \gamma \Delta_t^{u,(\ell)},
& \text{if } s_t^{(\ell)} > 0, \\[0.2em]
-\Delta_t^{u,(\ell)},
& \text{if } s_t^{(\ell)} \le 0,
\end{cases}
\label{eq:layerwise_rectify}
\end{equation}
where $\gamma$ controls the strength in suppressing the forget-stream direction. Finally, the server updates each layer by
\begin{equation}
\boldsymbol{w}_{t+1}^{(\ell)}
=
\boldsymbol{w}_{t}^{(\ell)}
+
\widetilde{\Delta}_t^{(\ell)}.
\label{eq:update}
\end{equation}

The above rectification rule is shared by different unlearning granularities, while the construction of the two streams depends on the unlearning request. 
For client-level unlearning, all data from the unlearning client are used to construct the forget stream, and data from the remaining clients are used to construct the retain stream. 
For sample-level unlearning, only the designated samples of the unlearning client are used for the forget stream, while the remaining data (including the rest of data on it and the data from other clients), contribute to the retain stream. 
For class-level unlearning, all samples belonging to the unlearning class across clients are used to construct the forget stream, and all non-target-class samples are used to construct the retain stream. 
Thus, different unlearning requests only change how the forget and retain streams are formed, while the layer-wise rectification mechanism remains the same.

\subsection{Phase 3: Restore Retained Utility}

After $T$ rectification rounds, the adapter $\boldsymbol{\phi}^{*}$ is removed. To recover retained utility, we further run $R$ rounds of FL on the retained data.

The pseudocode of \textsc{Lethe} is presented in Algorithm~\ref{alg:Lethe}. 

\subsection{Theoretical Analysis}
\label{sec:theory-analysis}
We analyze conflict-aware rectification, stability of the terminal recovery-alignment margin, and parameter rollback~\cite{wang2025invariance}. A finite path-length condition links local alignment stability to a conditional bound on forget-loss recovery over a prescribed number of retained-data FL rounds. The unlearning loss $\mathcal{L}_u$ and the recovery direction below are analytical quantities only; complete derivations are provided in Appendix S1.

\paragraph{Rectified unlearning step}
Let $\mathcal{L}_u$ and $\mathcal{L}_r$ be the unlearning and retained losses. The two local streams admit
\begin{equation}
\Delta_t^u=-\eta_u\nabla\mathcal{L}_u(\bw_t)+\boldsymbol e_t^u,
\qquad
\Delta_t^r=-\eta_r\nabla\mathcal{L}_r(\bw_t)+\boldsymbol e_t^r,
\label{eq:stream_gradient_approx}
\end{equation}
where the error terms collect local-training and aggregation effects. Hence, $-\Delta_t^u$ is a first-order forgetting direction.

\begin{proposition}[Conflict-aware rectification]
\label{prop:rectify_correct}
For any layer and $\gamma>0$, Eq.~\eqref{eq:layerwise_rectify} satisfies
\begin{equation}
\begin{aligned}
s_t^{(\ell)}\le0:\quad
&\langle\widetilde{\Delta}_t^{(\ell)},\Delta_t^{u,(\ell)}\rangle
=-\|\Delta_t^{u,(\ell)}\|^2,\\
&\langle\widetilde{\Delta}_t^{(\ell)},\Delta_t^{r,(\ell)}\rangle
=-s_t^{(\ell)}\ge0;\\
s_t^{(\ell)}>0:\quad
&\widetilde{\Delta}_t^{(\ell)}-\Delta_t^{r,(\ell)}
=-\gamma\Delta_t^{u,(\ell)},\\
&\langle\widetilde{\Delta}_t^{(\ell)},\Delta_t^{u,(\ell)}\rangle
=s_t^{(\ell)}-\gamma\|\Delta_t^{u,(\ell)}\|^2<s_t^{(\ell)}.
\end{aligned}
\label{eq:conflict_aware_rectification}
\end{equation}
\end{proposition}

\begin{proof}
See Appendix S1-A.
\end{proof}

The proposition implies that, when $s_t^{(\ell)}\le0$, the reverse forget update is aligned with the retain
stream and therefore preserves retained utility to first order under the
exact-gradient interpretation. When $s_t^{(\ell)}>0$, the retain stream
contains a target-refitting component; rectification injects
$-\gamma\Delta_t^{u,(\ell)}$, a reverse update on the deleted objective, and
reduces alignment with the forget-data descent direction by exactly
$\gamma\|\Delta_t^{u,(\ell)}\|^2$.

\paragraph{Post-unlearning analysis}
The following quantities only define notation to facilitate post-unlearning analysis;
they neither add an algorithmic step nor impose a condition. For any state
$\bw$, let $\mathcal{T}_r(\bw)$ be the state after one retained-data FL round,
and define
\begin{equation}
\begin{aligned}
\boldsymbol r(\bw)&=\mathcal{T}_r(\bw)-\bw,\\
\boldsymbol u(\bw)&=-\eta_u\nabla\mathcal{L}_u(\bw),\\
a(\bw)&=\langle\boldsymbol r(\bw),\boldsymbol u(\bw)\rangle,
\end{aligned}
\label{eq:post_maps}
\end{equation}

where $\boldsymbol r$ is the retained update, $\boldsymbol u$ is a virtual
forget-loss descent direction with $\eta_u>0$, and $a$ is their alignment.
The map-based results use a fixed deterministic $\mathcal{T}_r$ as an
analytical model. 
Starting from
$\bar{\bw}_0=\bw_{\mathrm{un}}$, continued training follows
\begin{equation}
\begin{aligned}
\bar{\bw}_{h+1}&=\bar{\bw}_h+\boldsymbol r_h,\\
\boldsymbol r_h&=\boldsymbol r(\bar{\bw}_h),\\
\boldsymbol u_h&=\boldsymbol u(\bar{\bw}_h),\qquad
a_h=a(\bar{\bw}_h).
\end{aligned}
\label{eq:continued_streams}
\end{equation}
Positive $a_h$ indicates a first-order target-refitting tendency in forget-loss
space. We define the corresponding loss recovery:
\begin{equation}
\mathsf{LR}_H=
\max_{0\le j\le H}
[\mathcal{L}_u(\bar{\bw}_0)-\mathcal{L}_u(\bar{\bw}_j)]_+.
\label{eq:forget_loss_recovery}
\end{equation}

Here, $\mathsf{LR}_H$ is loss-space recovery, for which a smaller value is preferable.
Note that our bounds concern $\mathsf{LR}_H$, not directly KR or RR, because lower loss changes KR only if predictions or decoded outputs change.

\begin{proposition}[Local stability of a \textsc{Lethe} output margin]
\label{prop:local_alignment_stability}
Suppose the returned model satisfies $a(\bar{\bw}_0)\le-m_a$ for some
$m_a>0$, and $\boldsymbol r$ and $\boldsymbol u$ are respectively
$K_r$- and $K_u$-Lipschitz on $\mathcal{B}_R(\bar{\bw}_0)$. For any
$\bw$ in this ball and $d=\|\bw-\bar{\bw}_0\|$,
\begin{equation}
\begin{aligned}
a(\bw)-a(\bar{\bw}_0)
&\le (K_r\|\boldsymbol u_0\|+K_u\|\boldsymbol r_0\|)d+K_rK_ud^2.
\end{aligned}
\label{eq:alignment_stability_bound}
\end{equation}
Let $C_R$ denote the right-hand side at $d=R$. If $C_R<m_a$, then
$a(\bw)<0$ throughout $\mathcal{B}_R(\bar{\bw}_0)$.
\end{proposition}

\begin{proof}
See Appendix S1-B.
\end{proof}

This conflict-aware construction alone does not guarantee $m_a>0$. The margin is an
output-level certificate for the returned model and is directly testable as
$-a(\bar{\bw}_0)/(\|\boldsymbol r_0\|\|\boldsymbol u_0\|)$.

\paragraph{Rollback under continued training}
Parameter rollback~\cite{wang2025invariance} is another perspective to interpret knowledge resurfacing. Let
$\boldsymbol d_h=\bw_{\mathrm{ori}}-\bar{\bw}_h$ and
$b_h=\langle\boldsymbol r_h,\boldsymbol d_h\rangle$.

\begin{proposition}[Parameter rollback]
\label{prop:rollback_identity}
Let
$D_h=\|\boldsymbol d_h\|
=\|\bw_{\mathrm{ori}}-\bar{\bw}_h\|$
denote the Euclidean distance between the model after $h$ rounds of continued
training and the pre-unlearning model $\bw_{\mathrm{ori}}$. Every
continued-training step satisfies
\begin{equation}
D_{h+1}^2
=D_h^2+\|\boldsymbol r_h\|^2-2b_h.
\label{eq:rollback_identity}
\end{equation}
Therefore, $b_h>\|\boldsymbol r_h\|^2/2$ implies $D_{h+1}<D_h$.
\end{proposition}

\begin{proof}
See Appendix S1-C.
\end{proof}

This proposition provides theoretical evidence that, under certain condition, continued training will gradually pull the model backward to its state before unlearning.

\begin{proposition}[Old-basin forget-loss recovery]
\label{prop:old_basin_return}
Let $\bw_r^{\mathrm{old}}=\lim_{j\to\infty}\mathcal{T}_r^j(\bw_{\mathrm{ori}})$
be the fixed point reached by repeatedly training the pre-unlearning model on
retained data. Its attraction basin under $\mathcal{T}_r$ is
\begin{equation}
\mathcal{A}_{\mathrm{old}}
=
\left\{
\bw:
\lim_{j\to\infty}\mathcal{T}_r^j(\bw)
=\bw_r^{\mathrm{old}}
\right\}.
\label{eq:old_basin_definition}
\end{equation}
Thus, $\mathcal{A}_{\mathrm{old}}$ contains the starting models that ordinary
retained-data training drives to the same fixed point as the pre-unlearning
model. Suppose $\bw_{\mathrm{ori}},\bar{\bw}_0\in
\mathcal{A}_{\mathrm{old}}$. Assume, for some $\kappa\in(0,1)$ and every
$\bw\in\mathcal{A}_{\mathrm{old}}$,
\begin{equation}
\|\mathcal{T}_r(\bw)-\bw_r^{\mathrm{old}}\|
\le\kappa\|\bw-\bw_r^{\mathrm{old}}\|,\qquad
\|\bw_r^{\mathrm{old}}-\bw_{\mathrm{ori}}\|\le\delta.
\label{eq:old_basin_contraction}
\end{equation}
If $\mathcal{L}_u$ is $G_{\mathrm{Lip}}$-Lipschitz around $\bw_{\mathrm{ori}}$ and the trajectory, then
\begin{equation}
\begin{aligned}
\|\bar{\bw}_h-\bw_{\mathrm{ori}}\|
&\le\kappa^h\|\bar{\bw}_0-\bw_r^{\mathrm{old}}\|+\delta,\\
\liminf_{H\to\infty}\mathsf{LR}_H
&\ge[\mathcal{L}_u(\bar{\bw}_0)-\mathcal{L}_u(\bw_{\mathrm{ori}})
-G_{\mathrm{Lip}}\delta]_+.
\end{aligned}
\label{eq:old_basin_result}
\end{equation}
\end{proposition}

\begin{proof}
See Appendix S1-D.
\end{proof}

Thus, when
$m_u=\mathcal{L}_u(\bw_{\mathrm{un}})-\mathcal{L}_u(\bw_{\mathrm{ori}})
>G_{\mathrm{Lip}}\delta$, a returned model that remains in
$\mathcal{A}_{\mathrm{old}}$ is pulled toward the old fixed point and has a
strictly positive forget-loss-recovery lower bound. This complements the preceding
alignment analysis: $a_0\approx0$ can coexist with later forget-loss recovery because
$a_0$ describes only the alignment of the first retained update, whereas the
basin condition determines the destination of the full continued-training
trajectory.

\paragraph{Finite-horizon forget-loss recovery}
We fix a finite integer $H\ge1$ and connect a generic smoothness bound to the
\textsc{Lethe} output margin and a path budget for the first $H$ rounds.

\begin{theorem}[Multi-round forget-loss-recovery bound]
\label{thm:forget_loss_recovery_bound}
If $\mathcal{L}_u$ has a $\beta_u$-Lipschitz gradient on a neighborhood
containing every segment $[\bar{\bw}_h,\bar{\bw}_{h+1}]$, $0\le h<H$, then
\begin{equation}
\mathsf{LR}_H
\le
\sum_{h=0}^{H-1}
\left(
\frac{[a_h]_+}{\eta_u}
+\frac{\beta_u}{2}\|\boldsymbol r_h\|^2
\right).
\label{eq:forget_loss_recovery_bound}
\end{equation}
\end{theorem}

\begin{proof}
See Appendix S1-E.
\end{proof}

\begin{samepage}
\begin{corollary}[Finite-horizon \textsc{Lethe} recovery bound]
\label{cor:lethe_finite_horizon}
Suppose Proposition~\ref{prop:local_alignment_stability} holds on
$\mathcal{B}_R(\bar{\bw}_0)$ with
$\varepsilon_a:=m_a-C_R>0$, and $\mathcal{L}_u$ is $\beta_u$-smooth on a
neighborhood of this ball. For the fixed horizon $H$, suppose nonnegative
update bounds $\rho_0,\ldots,\rho_{H-1}$ satisfy
\begin{equation}
\|\boldsymbol r_h\|\le\rho_h\quad(0\le h<H),
\qquad
\sum_{h=0}^{H-1}\rho_h\le R.
\label{eq:lethe_path_conditions}
\end{equation}
Then
\begin{equation}
\begin{aligned}
\mathsf{LR}_H
&\le\max_{1\le j\le H}
\left[\frac{\beta_u}{2}\sum_{h=0}^{j-1}\rho_h^2
-\frac{j\varepsilon_a}{\eta_u}\right]_+\\
&\le\frac{\beta_u}{2}\sum_{h=0}^{H-1}\rho_h^2,
\end{aligned}
\label{eq:lethe_finite_horizon_bound}
\end{equation}
\end{corollary}
\end{samepage}

\begin{proof}
See Appendix S1-F.
\end{proof}

The path budget keeps the trajectory through $H$ rounds inside the certified
ball, where $a_h\le-\varepsilon_a$. The first bound retains this margin;
the second uses only squared update bounds. Appendix S1-F gives a
geometric envelope and a nonconstant quadratic example satisfying all premises.

Proposition~\ref{prop:rectify_correct} establishes the one-step rectification
effect, while
Theorem~\ref{thm:forget_loss_recovery_bound} and
Corollary~\ref{cor:lethe_finite_horizon} give conditional finite-horizon
forget-loss bounds. Proposition~\ref{prop:old_basin_return} describes recovery upon return to the old basin. 

\begin{figure}[t]
    \centering
    \includegraphics[width=\linewidth]{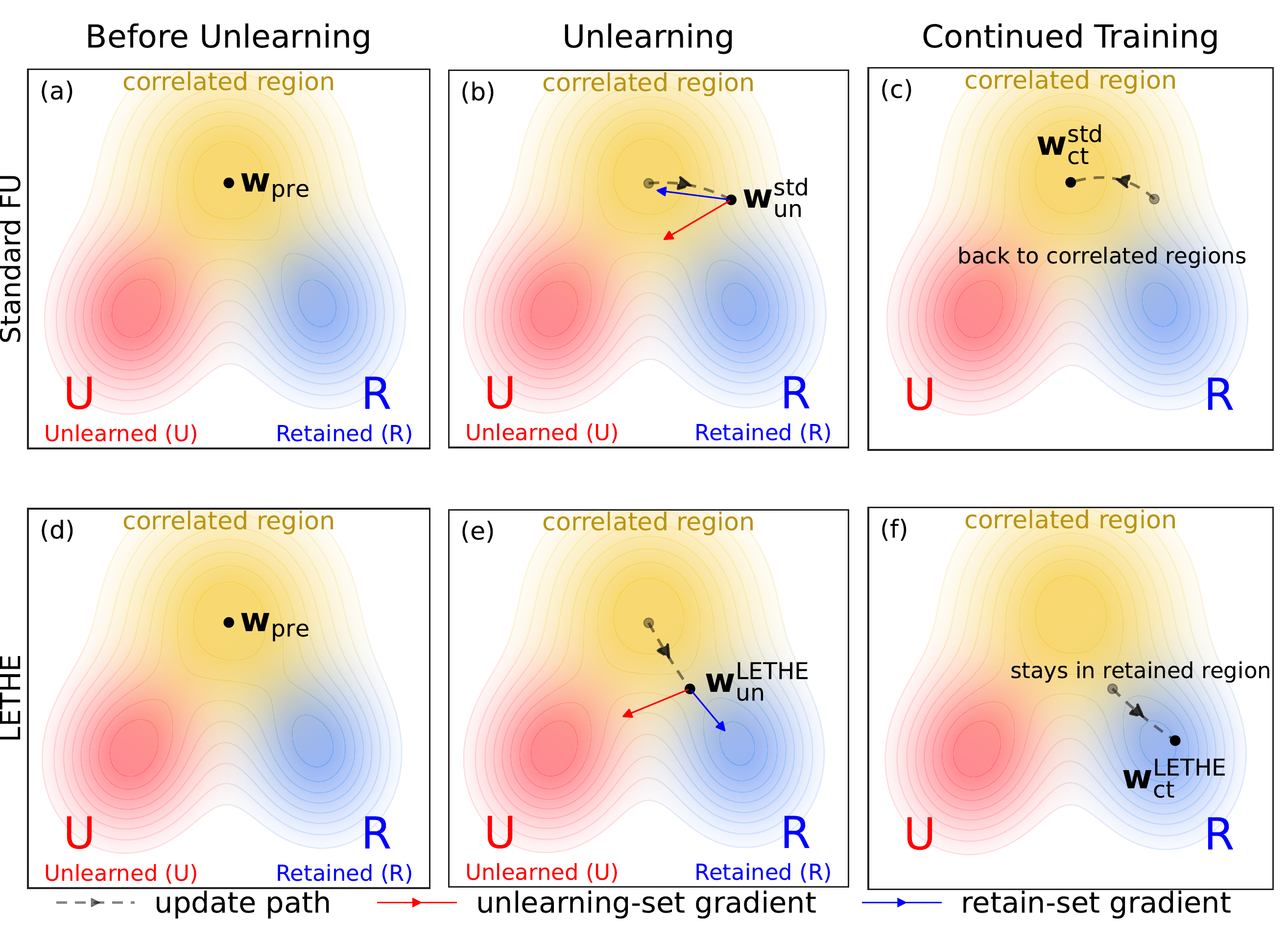}
\caption{Conceptual parameter-space illustration. Panels (a) and (d) show the shared 
pre-unlearning state and the correlated region connecting the unlearning and retained regions.
 During unlearning, standard FU may remain near this region (b), whereas \textsc{Lethe} is 
 depicted farther from it after dual-stream update rectification (e). During continued training,
  standard FU may roll back toward the correlated region (c), whereas \textsc{Lethe} tends 
  to remain in the retained region and preserve unlearning (f).}
    \label{fig:region}
\end{figure}

\subsection{Conceptual Interpretation}
\label{sec:conceptual_interpretation}

In Fig.~\ref{fig:region} we visualize how the dual-stream rectification works in the parameter space. 
Before unlearning, the model resides in a correlated region where the loss on both unlearning set and retain set is low (1st column). Existing FU methods can suppress the target behavior immediately after unlearning (panel b), but the model ``rolls back'' during follow-up training (panel c). 
By contrast, \textsc{Lethe} explicitly rectifies the update trajectory during unlearning. \textsc{Lethe} redirects the path of optimization to consistently induce forgetting while preserving utility. By suppressing retain-stream components that are positively aligned with the forget stream, \textsc{Lethe} moves the model away from the correlated region to keep unlearned knowledge de-activated (panels e\&f). 

\subsection{Threat Model and Privacy.}
\textsc{Lethe} assumes non-colluding clients and a protocol-following honest-but-curious server. The forget stream \(\Delta_t^{u}\) and retain stream \(\Delta_t^{r}\) are conventional backbone updates computed on the unlearning data \(\mathcal{D}_u\) and retained data, respectively. Both streams can be transmitted over secure channels and protected by homomorphic encryption mechanisms. One can also employ secure aggregation to protect individual retained-client updates when forming \(\Delta_t^{r}\), whereas the rectification operations during Phase~2 can take place in a trusted execution environment for security.

\section{Experiments}
\label{sec:exp}

\subsection{Experimental Setting}
\label{sec:exp_setting}
We evaluate \textsc{Lethe} on CV and NLP tasks using diverse models, datasets,
federated settings, unlearning targets, baselines, and evaluation metrics. Complete federation, training, and implementation details are
provided in Appendices S2-A and S2-B.

\paragraph{Models and datasets}
We evaluate \textsc{Lethe} under both CV and LLM-based NLP unlearning settings.
For CV unlearning, we consider image classification using MNIST with LeNet-5~\cite{lecun1998gradient}, CIFAR-10~\cite{krizhevsky2009learning} with ResNet-18~\cite{he2016deep}, and Tiny-ImageNet~\cite{tinyimagenet} with HSViT~\cite{xu2024hsvit}.
For LLM-based NLP unlearning, we use TOFU~\cite{tofu2024} with Llama-3.2-3B-Instruct~\cite{meta2024llama32modelcard}, as well as LUME Task 2~\cite{ramakrishna-etal-2025-lume}, which contains synthetic biographies with sensitive personally identifiable information, with OLMo-1B~\cite{groeneveld-etal-2024-olmo}.

\paragraph{Federated setting and unlearning targets}
We evaluate whether an unlearning method can (i) effectively remove the target knowledge associated with the unlearning data and (ii) maintain the effect of unlearning under continued training, thereby resisting \emph{knowledge resurfacing}. 
Following prior FU evaluations~\cite{fedau,pan2025fedosd}, we adopt a Non-IID federated setting.

\textbf{CV task settings.}
For CV tasks, we partition the training data into $K=50$ clients using a Dirichlet distribution $\mathrm{Dir}(\alpha)$ with concentration parameter $\alpha=0.1$, which yields heterogeneous data distributions~\cite{fedau,pan2025fedosd}. 
We evaluate sample-, class-, and client-level unlearning. 

For client-level and sample-level unlearning, we consider both trigger-based and non-trigger-based settings. Following prior studies~\cite{fedau,pan2025fedosd}, we adopt the trigger-based construction to explicitly instantiate the target knowledge to be removed, while additionally evaluating \textsc{Lethe} under non-trigger-based settings to assess its effectiveness in more natural unlearning scenarios. For client-level unlearning, the data of the target client are treated as the unlearning set. In the trigger-based setting, a shared trigger is injected into these samples, whereas in the non-trigger-based setting, the original data are directly used. For sample-level unlearning, we randomly select $20\%$ of the training samples from one client as the unlearning set. For class-level unlearning, the objective is to remove the knowledge associated with the target class.

\textbf{NLP task settings.}
We evaluate sample-level LLM unlearning on TOFU~\cite{tofu2024} with
Llama-3.2-3B-Instruct~\cite{meta2024llama32modelcard} and LUME Task
2~\cite{ramakrishna-etal-2025-lume} with OLMo-1B~\cite{groeneveld-etal-2024-olmo}.
Each uses $K=20$ clients, with $20\%$ of client 0's samples for unlearning and
all others for retention. Both models use LoRA~\cite{hu2022lora}.

\begin{table}[t]
    \centering
    \caption{Performance comparison under client-level unlearning. 
    \textbf{MU} is higher better, while \textbf{U-KR}, \textbf{CT-KR}, and \textbf{RR} are lower better. 
    Best and second-best results are marked in \textbf{bold} and \underline{underline}, respectively.}
    \label{tab:client_unlearning}
    \scriptsize
    \setlength{\tabcolsep}{3pt}
    \renewcommand{\arraystretch}{0.95}

    \resizebox{\columnwidth}{!}{
    \begin{tabular}{llcccc}
        \toprule
        \textbf{Setting} 
        & \textbf{Method}
        & \textbf{MU} ($\uparrow$)
        & \textbf{U-KR} ($\downarrow$)
        & \textbf{CT-KR} ($\downarrow$)
        & \textbf{RR} ($\downarrow$) \\
        \midrule

        \multirow{7}{*}{LeNet5}
        & Origin    & 97.86 & 98.99 & ---   & --- \\
        & Retrain   & 97.59 & 0.31  & 0.56  & --- \\
        \noalign{\vskip 1pt}
        \cdashline{2-6}
        \noalign{\vskip 2pt}

        & FedEraser & 97.93 
                    & \underline{0.25} 
                    & \textbf{0.19} 
                    & \textbf{0.00} \\

        & FedAU     & \underline{98.34} 
                    & 0.62 
                    & 49.37 
                    & 49.59 \\

        & FedOSD    & \textbf{98.45} 
                    & 0.65 
                    & 4.18 
                    & \underline{3.68} \\

        & NoT       & 97.73 
                    & 1.02 
                    & 0.37 
                    & \textbf{0.00} \\

        & \textsc{Lethe}
                    & 98.03
                    & \cellcolor{letheblue}{\textbf{0.22}}
                    & \cellcolor{letheblue}{\underline{0.34}}
                    & \cellcolor{letheblue}{\textbf{0.00}} \\
        \midrule

        \multirow{7}{*}{ResNet18}
        & Origin    & 59.97 & 90.77 & ---   & --- \\
        & Retrain   & 58.72 & 4.00  & 3.22  & --- \\
        \noalign{\vskip 1pt}
        \cdashline{2-6}
        \noalign{\vskip 2pt}

        & FedEraser & \underline{64.45}
                    & 0.22 
                    & \underline{2.11} 
                    & \textbf{0.00} \\

        & FedAU     & 58.26
                    & \textbf{0.00} 
                    & 90.63 
                    & 99.84 \\

        & FedOSD    & \textbf{65.93} 
                    & 0.56 
                    & 67.48 
                    & 73.40 \\

        & NoT       & 57.73 
                    & 35.78 
                    & 20.78 
                    & UF \\

        & \textsc{Lethe}
                    & 57.84
                    & \cellcolor{letheblue}{\underline{0.11}}
                    & \cellcolor{letheblue}{\textbf{1.15}}
                    & \cellcolor{letheblue}{\textbf{0.00}} \\
        \midrule

        \multirow{6}{*}{HSViT}
        & Origin    & 48.73 & 91.65 & ---   & --- \\
        & Retrain   & 49.85 & 0.47  & 0.42  & --- \\
        \noalign{\vskip 1pt}
        \cdashline{2-6}
        \noalign{\vskip 2pt}

        & FedAU     & \textbf{49.67} 
                    & \underline{0.35} 
                    & 8.74 
                    & 9.12 \\

        & FedOSD    & 10.66 
                    & 0.36 
                    & 36.53 
                    & 39.58 \\

        & NoT       & 48.46 
                    & 0.55 
                    & \underline{0.60} 
                    & \underline{0.20} \\

        & \textsc{Lethe}
                    & \underline{49.29}
                    & \cellcolor{letheblue}{\textbf{0.10}}
                    & \cellcolor{letheblue}{\textbf{0.45}}
                    & \cellcolor{letheblue}{\textbf{0.03}} \\
        \midrule

        \multirow{6}{*}{
            \shortstack{ResNet18\\(Non-trigger)}
        }
        & Origin    & 68.36 & 96.60 & ---   & --- \\
        & Retrain   & 65.95 & 48.27 & 48.30 & --- \\
        \noalign{\vskip 1pt}
        \cdashline{2-6}
        \noalign{\vskip 2pt}

        & FedAU     & \underline{68.91} 
                    & 68.69 
                    & 79.46 
                    & 64.51 \\

        & FedOSD    & 67.74 
                    & 61.29 
                    & 76.20 
                    & 57.76 \\

        & NoT       & 64.33 
                    & \underline{47.92} 
                    & \underline{48.18} 
                    & \textbf{0.00} \\

        & \textsc{Lethe}
                    & \textbf{69.08}
                    & \cellcolor{letheblue}{\textbf{47.67}}
                    & \cellcolor{letheblue}{\textbf{47.78}}
                    & \cellcolor{letheblue}{\textbf{0.00}} \\
        \bottomrule
    \end{tabular}
    }

    \vspace{0.5mm}
    \footnotesize{
    \textit{Note:} UF denotes unsuccessful forgetting.
    }

\end{table}

\paragraph{Baselines}
We first train an \textbf{Origin} model on the full training data as the common starting point for all unlearning methods. We also include \textbf{Retrain}, which rebuilds the model from scratch on the retained data using standard federated learning, as the reference for optimal unlearning. Since \textsc{Lethe} is an approximate FU method balancing unlearning effectiveness and retained utility, we primarily compare it with representative approximate unlearning methods.

For CV tasks, we compare \textsc{Lethe} with \textbf{FedEraser}~\cite{federaser}, \textbf{FedAU}~\cite{fedau}, \textbf{FedOSD}~\cite{pan2025fedosd}, and \textbf{NoT}~\cite{fednot}. 
\textbf{FedEraser} accelerates retraining by reusing historical local updates, but it is evaluated only for client-level unlearning because it relies on stored historical updates. 
We omit FedEraser on Tiny-ImageNet with HSViT due to the prohibitive storage overhead of saving historical HSViT updates across rounds. 
{FedAU} and {NoT} support sample-, class-, and client-level unlearning, while {FedOSD} is evaluated only for client-level unlearning.

For NLP tasks, following Zhang et al.~\cite{zhang2026oblivionis}, 
we include \textbf{NPO}~\cite{NPO} and \textbf{SimNPO}~\cite{simNPO}, 
which are adopted in Oblivionis~\cite{zhang2026oblivionis} as unlearning baselines for federated LLMs. 
For both methods, the corresponding unlearning objective is optimized locally, 
and the resulting LoRA updates are uploaded to the server for aggregation. 
In our implementation, NPO and SimNPO require only one communication round 
to complete the unlearning process.

\paragraph{Evaluation metrics}
We evaluate each method in model utility, knowledge retention, resurfacing rate, communication cost, and computational overhead. 

\textbf{Model Utility (MU)} measures the retained utility of the model
(the higher the better).
For CV tasks, MU is instantiated as the clean test accuracy.
For NLP tasks, MU is instantiated as the mean ROUGE-L F1 on the corresponding
retained evaluation set.

\textbf{Unlearning-phase Knowledge Retention (U-KR)} measures the target knowledge retention immediately after unlearning. For generative LLM tasks, it instantiates Eq.~\eqref{eq:llm_kr}. A lower value indicates a stronger immediate unlearning effect.

\textbf{Continued-Training Knowledge Retention (CT-KR)} is the target-set KR at the final checkpoint $T$ after continued training on the retraining baseline's retained dataset. Lower values indicate stronger resistance to knowledge recovery. The LLM decoding configuration is fixed across evaluation points.

\textbf{Resurfacing Rate (RR)}, defined in Eq.~\eqref{eq:rr}, measures relative endpoint recovery beyond clean retraining, with lower values indicating stronger unlearning persistence. If a method yields high U-KR, it fails to remove the target knowledge immediately after unlearning; in this case, RR is not meaningful and the result is marked as unsuccessful forgetting (UF).

\textbf{Communication Cost (CM)} measures the total volume of model updates transmitted between clients and the server during unlearning.

\textbf{Floating-Point Operations (FLOPs)} measures the computation overhead of the unlearning process.

Importantly, \textit{U-KR alone only reflects the immediate effect of unlearning} and may lead to an \textit{illusion of unlearning}; therefore, CT-KR and RR are necessary for evaluating unlearning persistence under continued training.

\begin{figure*}[t]
    \centering
    \includegraphics[width=\textwidth]{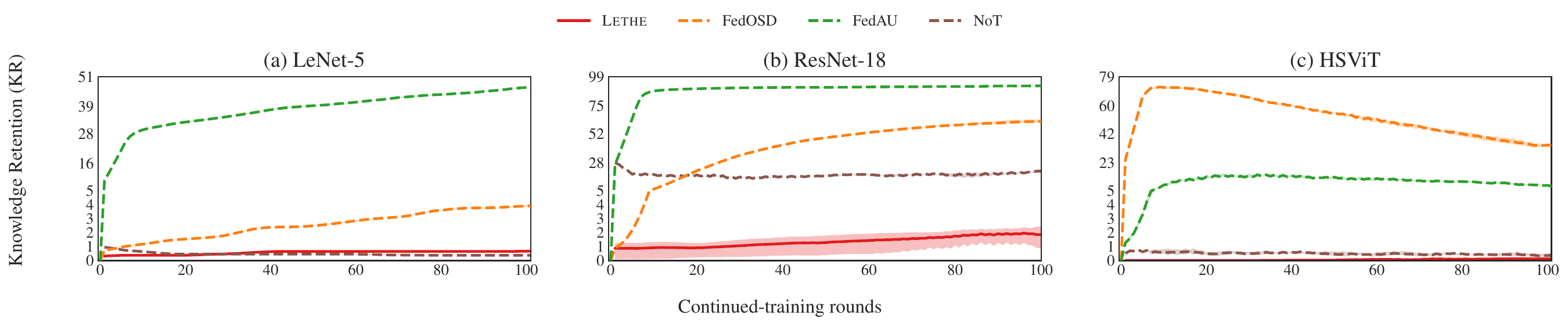}
    \caption{Undesired knowledge retention (KR) during continued training on retained data. Curves show mean KR 
    with one standard deviation. \textsc{Lethe} consistently maintains low KR across all three
     models, whereas competing methods exhibit varying degrees of knowledge resurfacing.}
     \label{fig:kr_continued}
\end{figure*}

\begin{figure}[t]
    \centering
    \includegraphics[width=\linewidth]{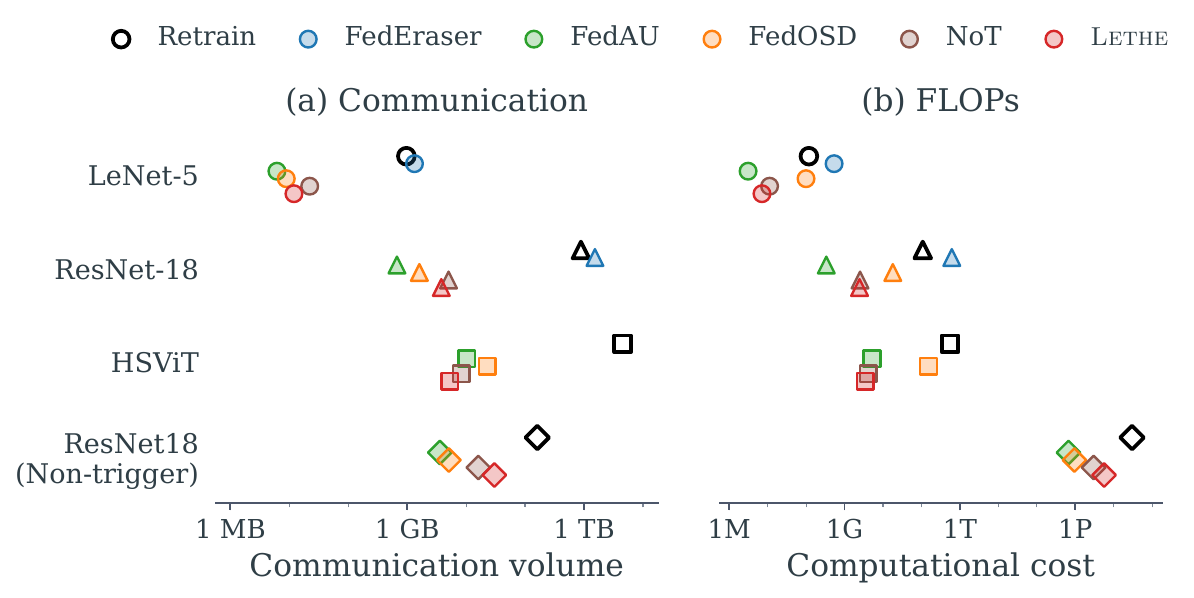}
    \caption{Efficiency comparison on CV tasks. Panels (a) and (b) report
    communication volume and computational cost in FLOPs, respectively.
    \textsc{Lethe} incurs substantially lower communication and computation
    costs than Retrain and FedEraser across the evaluated settings.}
    \label{fig:cm_flops}
\end{figure}

\subsection{Results on CV Tasks}
\label{sec:cv_tasks}
We evaluate client-, class-, and sample-level CV unlearning before comparing
the persistence and retained utility of the resulting models.

\subsubsection{Client-Level unlearning}
\label{sec:client_unlearning}

Table~\ref{tab:client_unlearning} reports the client-level unlearning results on LeNet-5,
 ResNet-18, and HSViT. 
Overall, \textsc{Lethe} consistently achieves strong unlearning persistence under continued 
training while maintaining competitive model utility. 
Across the three backbones, \textsc{Lethe} keeps CT-KR close to the clean retraining 
baseline, whereas several baselines exhibit clear knowledge resurfacing after continued 
training. 
In particular, FedAU achieves low U-KR immediately after unlearning but show much higher 
CT-KR later, which justifies that U-KR (i.e., u-set accuracy in the literature) alone 
cannot reveal the persistence of unlearning.

\textbf{Persistence under continued training:}
Fig.~\ref{fig:kr_continued} tracks the level of knowledge retention (KR) during continued training on retained data. 
A clear pattern is that methods suffering from knowledge resurfacing often exhibit a sharp increase in KR within the first few rounds after continued training starts. 
For example, FedOSD rapidly recovers target knowledge on HSViT, while FedAU shows strong early-stage resurfacing on LeNet-5 and ResNet-18. 
This suggests that removed knowledge can be reactivated very early once retained-data training resumes. 
In contrast, \textsc{Lethe} shows no noticeable increase in KR across all three models, indicating strong unlearning persistence.

\textbf{Efficiency analysis:}
Fig.~\ref{fig:cm_flops} compares the communication and computation costs of different methods. Panels (a) and (b) report communication cost and FLOPs, respectively. \textsc{Lethe} substantially reduces both costs compared with full retraining and FedEraser without storing historical updates. End-to-end accounting, including the temporary-adapter overhead, is detailed in Appendix S3-C. Although some lightweight baselines are cheaper in certain settings, they often exhibit higher CT-KR or RR (Table~\ref{tab:client_unlearning}). Overall, \textsc{Lethe} achieves a favorable efficiency--persistence trade-off.

\begin{table}[tb]
    \centering
    \caption{Performance comparison on class- and sample-level unlearning.}
    \label{tab:class_unlearning}
    \label{tab:sample_unlearning}
    \scriptsize
    \setlength{\tabcolsep}{3pt}
    \renewcommand{\arraystretch}{1.05}
    \resizebox{\columnwidth}{!}{
    \begin{tabular}{llcccc}
        \toprule
        \textbf{Setting}
        & \textbf{Method} 
        & \textbf{MU} ($\uparrow$) 
        & \textbf{U-KR} ($\downarrow$) 
        & \textbf{CT-KR} ($\downarrow$) 
        & \textbf{RR} ($\downarrow$) \\
        \midrule
        \multirow{5}{*}{\shortstack[l]{Class-level\\unlearning}}
        & Origin  & 67.64 & 100.00 & ---  & --- \\
        & Retrain & 67.64 & \textbf{0.00}   & \textbf{0.00} & \textbf{0.00} \\
        \noalign{\vskip 2pt}
        \cdashline{2-6}
        \noalign{\vskip 2pt}
        & FedAU   & $\underline{67.87}$ & $\mathbf{0.00}$ & $\mathbf{0.00}$ & $\mathbf{0.00}$ \\
        & NoT     & $\mathbf{69.27}$ & $\mathbf{0.00}$ & $\mathbf{0.00}$ & $\mathbf{0.00}$ \\
        & \textsc{Lethe} & 67.01 & \cellcolor{letheblue}$\mathbf{0.00}$ & \cellcolor{letheblue}$\mathbf{0.00}$ & \cellcolor{letheblue}$\mathbf{0.00}$ \\
        \midrule
        \multirow{5}{*}{\shortstack[l]{Sample-level\\unlearning}}
        & Origin  & 66.33 & 82.80 & ---   & --- \\
        & Retrain & 64.65 & 2.40  & 2.35  & 0.00 \\
        \noalign{\vskip 2pt}
        \cdashline{2-6}
        \noalign{\vskip 2pt}
        & FedAU   & 53.40 & $\mathbf{0.40}$ & 61.80 & 73.90 \\
        & NoT     & $\underline{65.14}$ & 2.80 & $\underline{12.20}$ & $\underline{12.24}$ \\
        & \textsc{Lethe} & $\mathbf{66.30}$
        & \cellcolor{letheblue}{$\underline{0.80}$}
        & \cellcolor{letheblue}{$\mathbf{1.40}$} 
        & \cellcolor{letheblue}{$\mathbf{0.00}$} \\
        \bottomrule
    \end{tabular}
    }
\end{table}

\subsubsection{Class-Level Unlearning}
\label{sec:class_unlearning}
Table~\ref{tab:class_unlearning} reports the class-level unlearning results. 
Compared with client- and sample-level unlearning, class-level removal is relatively easier in this setting. 
All evaluated methods achieve near-perfect removal of the target class, with near-zero U-KR, CT-KR, and RR. 
This suggests that when the target knowledge is clearly separated by class labels, existing unlearning methods can remove it more reliably. 
Nevertheless, \textsc{Lethe} maintains competitive utility and comparable unlearning performance, showing its applicability across different CV unlearning granularities.

\subsubsection{Sample-Level Unlearning}
\label{sec:sample_unlearning}

Table~\ref{tab:sample_unlearning} reports the sample-level unlearning results. 
\textsc{Lethe} achieves the best overall trade-off between retained utility and unlearning persistence. 
Although some baselines can suppress target knowledge immediately after unlearning, their CT-KR increases markedly during continued training, showing that the unlearning effect is not persistent. 
In contrast, \textsc{Lethe} maintains low CT-KR and RR while preserving high MU, indicating that its unlearning effect remains stable after subsequent retained-data training.

\begin{figure}[t]
    \centering
    \includegraphics[width=\linewidth]{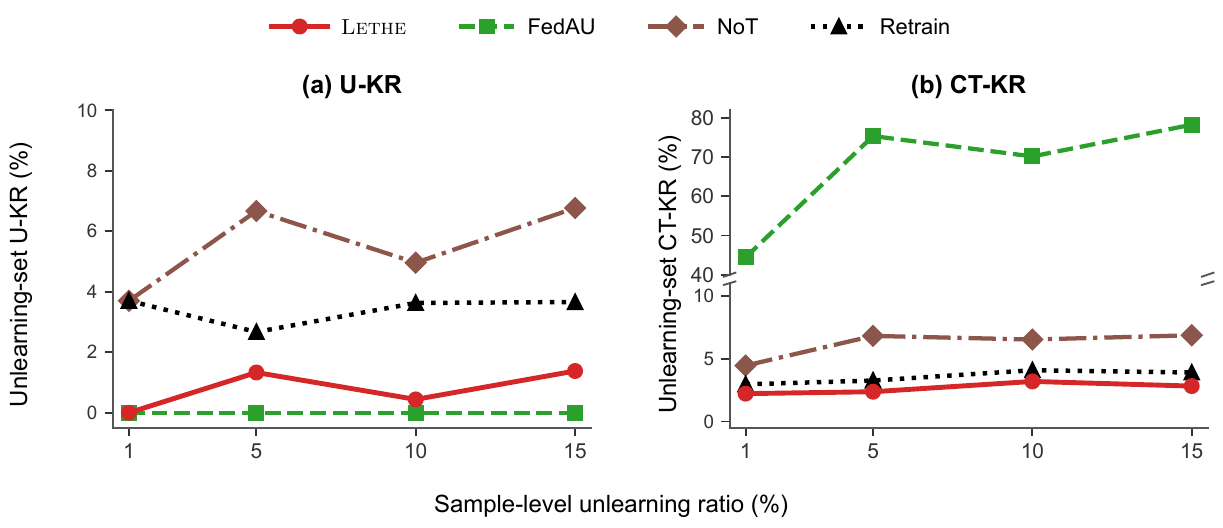}
    \caption{Sensitivity to sample-level unlearning ratios.
    We evaluate ratios of $1\%$, $5\%$, $10\%$, and $15\%$, and report
    (a) U-KR immediately after unlearning and (b) CT-KR after retained-data
    continued training. \textsc{Lethe} maintains low U-KR and keeps CT-KR close
    to the Retrain baseline across all ratios.}
   \label{fig:sample_ratio}
\end{figure}

\textbf{Sensitivity to sample-level unlearning ratio:}
We further evaluate robustness at sample-level unlearning ratios of $1\%$, $5\%$,
$10\%$, and $15\%$, reporting U-KR and CT-KR in
Fig.~\ref{fig:sample_ratio}. FedAU achieves near-zero U-KR at every ratio, but
its CT-KR remains substantially above Retrain, revealing strong knowledge
resurfacing during continued training. NoT exhibits higher U-KR and CT-KR than
\textsc{Lethe} across all ratios. In contrast, \textsc{Lethe} consistently
maintains low U-KR while keeping CT-KR close to the Retrain baseline,
demonstrating robust immediate forgetting and persistence across different
sample-level removal scales.

\begin{table}[t]
    \centering
    \caption{Performance comparison on LLM unlearning. Higher \textbf{MU} is
    better, while lower \textbf{U-KR}, \textbf{CT-KR}, and \textbf{RR} are
    better.}
    \label{tab:llm_unlearning}
    \scriptsize
    \setlength{\tabcolsep}{3pt}
    \renewcommand{\arraystretch}{0.95}

    \resizebox{\columnwidth}{!}{
    \begin{tabular}{llcccc}
        \toprule
        \textbf{Setting}
        & \textbf{Method}
        & \textbf{MU} ($\uparrow$)
        & \textbf{U-KR} ($\downarrow$)
        & \textbf{CT-KR} ($\downarrow$)
        & \textbf{RR} ($\downarrow$) \\
        \midrule

        \multirow{5}{*}{
            \shortstack{Llama-3.2-\\3B-Instruct}
        }
        & Origin
        & 0.5396
        & 0.4837
        & ---
        & --- \\

        & Retrain
        & 0.5385
        & 0.3470
        & 0.3967
        & 0.00 \\

        \noalign{\vskip 1pt}
        \cdashline{2-6}
        \noalign{\vskip 2pt}

        & NPO
        & 0.5114
        & 0.3592
        & $\underline{0.4537}$
        & $\underline{65.52}$ \\

        & SimNPO
        & $\mathbf{0.5513}$
        & $\underline{0.3588}$
        & 0.4684
        & 82.41 \\

        & \textsc{Lethe}
        & $\underline{0.5430}$
        & \cellcolor{letheblue}{$\mathbf{0.3345}$}
        & \cellcolor{letheblue}{$\mathbf{0.3957}$}
        & \cellcolor{letheblue}{$\mathbf{0.00}$} \\

        \midrule

        \multirow{5}{*}{OLMo-1B}
        & Origin
        & 1.0000
        & 1.0000
        & ---
        & --- \\

        & Retrain
        & 1.0000
        & 0.5500
        & 0.5459
        & 0.00 \\

        \noalign{\vskip 1pt}
        \cdashline{2-6}
        \noalign{\vskip 2pt}

        & NPO
        & $\mathbf{1.0000}$
        & $\mathbf{0.5450}$
        & 0.7294
        & 40.41 \\

        & SimNPO
        & $\mathbf{1.0000}$
        & 0.5881
        & $\underline{0.6834}$
        & $\underline{30.28}$ \\

        & \textsc{Lethe}
        & $\mathbf{1.0000}$
        & \cellcolor{letheblue}{$\underline{0.5462}$}
        & \cellcolor{letheblue}{$\mathbf{0.5327}$}
        & \cellcolor{letheblue}{$\mathbf{0.00}$} \\

        \bottomrule
    \end{tabular}
    }
\end{table}

\subsection{Results on NLP Tasks}
For NLP tasks, Table~\ref{tab:llm_unlearning} shows that \textsc{Lethe} achieves stronger persistence than NPO and SimNPO. 
Due to the semantic generalization ability of LLMs, KR may also increase for the clean retraining baseline after continued training; thus, Retrain serves as a reference for unavoidable generalization-induced recovery. 
Compared with NPO and SimNPO, \textsc{Lethe} keeps CT-KR closer to the retraining reference while maintaining competitive MU, indicating stronger unlearning persistence. 
Although NPO and SimNPO have lower CM and FLOPs due to restricted scope of client engagement, they show weaker persistence. 
In contrast, \textsc{Lethe} significantly stronger persistence (zero RR) at a cost much lower than retraining.

\begin{figure}[t]
    \centering
    \includegraphics[width=\linewidth]{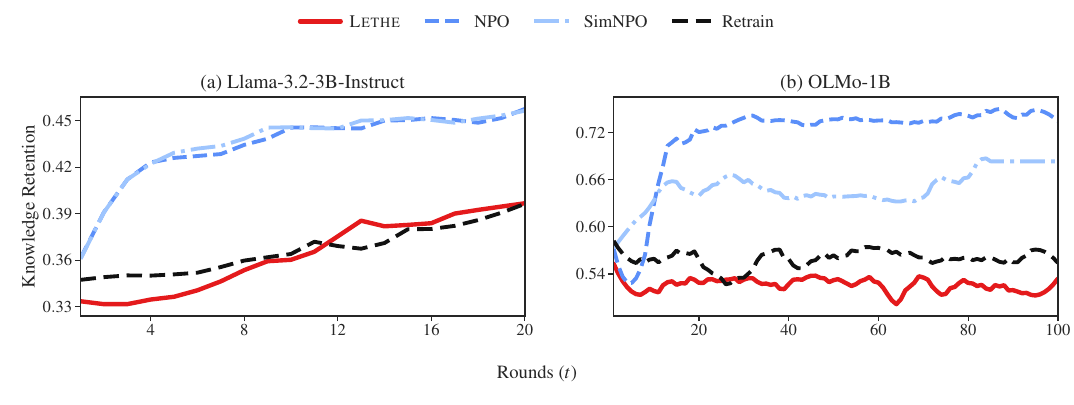}
    \caption{Persistence analysis on NLP tasks during continued training.
    We report the KR curves after unlearning.
    Compared with NPO and simNPO, \textsc{Lethe} remains close to the Retrain reference and exhibits a more stable KR trajectory, indicating stronger unlearning persistence.}
    \label{fig:nlp_ftr_curves}
\end{figure}

\begin{figure}[t]
    \centering
    \includegraphics[width=\columnwidth]{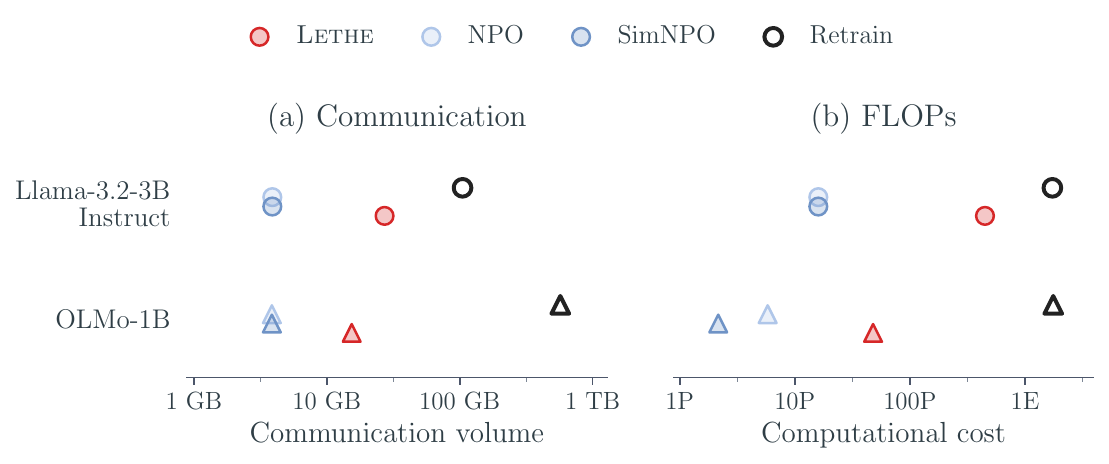}
    \caption{Efficiency comparison on NLP  tasks.
    (a) Communication volume and (b) computation cost measured in FLOPs.
    \textsc{Lethe} requires substantially lower communication and computation overhead than full retraining while maintaining effective and persistent unlearning across Llama-3.2-3B-Instruct and OLMo-1B.}
    \label{fig:llm_cm_flops}
\end{figure}

\textbf{Persistence under continued training:}
Fig.~\ref{fig:nlp_ftr_curves} tracks the level of KR during continued training on retained data. A clear pattern is that methods suffering from knowledge resurfacing exhibit a sharp increase in KR within the first few rounds. In particular, NPO and SimNPO rapidly recover target knowledge on both language models, indicating severe early-stage resurfacing. In contrast, \textsc{Lethe} maintains a substantially more stable trajectory and remains close to the clean retraining reference across both models, demonstrating strong unlearning persistence. The modest fluctuations of Retrain can be attributed to semantic generalization on retained data rather than the direct recovery of forgotten knowledge.

\textbf{Efficiency analysis:}
Although \textsc{Lethe} incurs higher communication and computation costs than lightweight baselines such as NPO and SimNPO, this overhead is acceptable in cross-silo federated settings, where the number of participating organizations is typically limited and unlearning is invoked only when deletion requests arise rather than in every training round. Moreover, the additional cost is confined to the unlearning stage and does not persist during subsequent federated training. In practice, the overhead can be further reduced through parameter-efficient updates, compressed communication, or fewer rectification and recovery rounds. Therefore, the additional cost of \textsc{Lethe} represents a deliberate trade-off for improving unlearning persistence rather than a recurring deployment overhead.

\subsection{Low-level Examination of Knowledge Resurfacing}
\label{sec:analysis}
Next, we analyze why knowledge resurfacing occurs and why \textsc{Lethe} mitigates it from three microscopic perspectives. 
Without loss of generality, we characterize knowledge resurfacing on CV tasks. Following prior FU evaluations~\cite{fedau,pan2025fedosd}, we use a trigger-based construction to instantiate target knowledge.

\textbf{Representation structure after unlearning:}
Following the representation-space diagnostic used by Tang et al.~\cite{PCA}, we project learned features into two dimensions with Principal Component Analysis (PCA)~\cite{jolliffe2016pca} to visualize how the target-related representation distribution changes after unlearning.
Specifically, we extract the penultimate-layer feature representations of unlearning-set samples and project them onto the first two principal components. We further report Forgetting Dispersion (FD)~\cite{NEURIPS2024_16e18fa3}, defined as the average squared distance between each unlearning sample representation and the feature center of the unlearning set. A larger FD indicates a more dispersed unlearning-set representation, while a smaller FD indicates a more compact representation. 

As shown in Fig.~\ref{fig:pca}, the evaluated baselines produce relatively compact trigger-related feature clusters and lower FD, whereas retraining and \textsc{Lethe} produce more dispersed distributions with higher FD. These observations describe the geometry captured by this projection and FD; they do not by themselves establish removal of the target knowledge.

\begin{figure}[t]
    \centering
    \includegraphics[width=\linewidth]{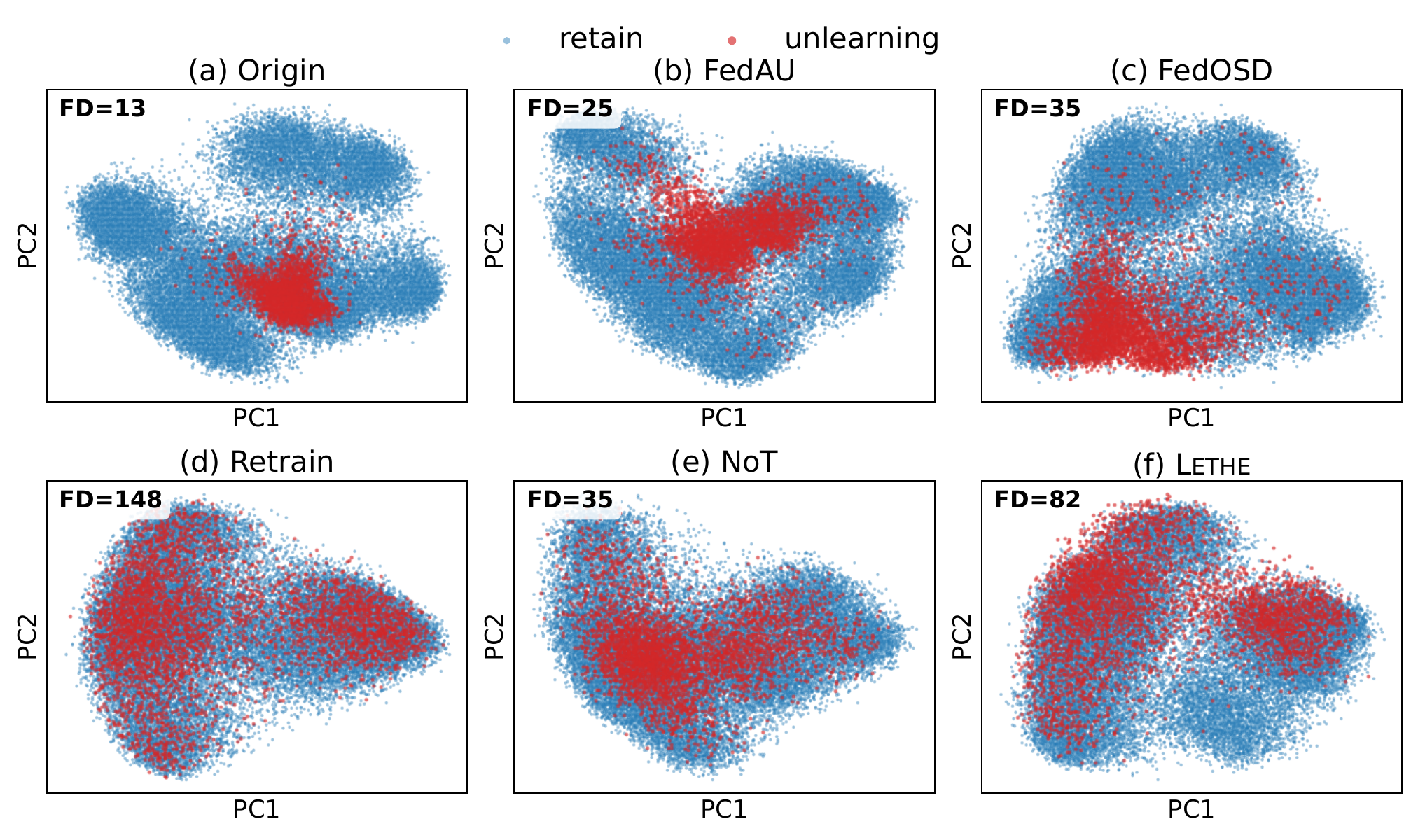}
    \caption{PCA visualization of retained samples and triggered unlearning-set samples immediately after unlearning. 
    (a) The original model forms a compact trigger-induced cluster. 
    (b), (c), and (e) Existing FU baselines retain noticeable clustering in this diagnostic. 
    (d) Retraining provides the dispersed gold-standard reference. 
    (f) \textsc{Lethe} produces a more dispersed distribution closer to retraining in the same diagnostic.}
    \label{fig:pca}
\end{figure}

\textbf{Rollback during continued training:}
Motivated by cosine-based task-vector diagnostics~\cite{wang2025invariance}, we examine whether continued training tends to move the unlearned model back toward the pre-unlearning model. 
Let $\bw_{\mathrm{ori}}$ and $\bw_{\mathrm{un}}$ denote the original global model and the unlearned model, respectively. 
We define $\bw_{\mathrm{ori}}-\bw_{\mathrm{un}}$ as the \emph{rollback vector}. 
At round $t$ of continued training, let $\Delta w_t$ be the global update induced by retained-data training. 
We define the \emph{rollback tendency (RT)} as the alignment between $\Delta w_t$ and the rollback vector:
\begin{equation}
\mathrm{RT}(t) \triangleq
\frac{
\langle \Delta \bw_t,\; \bw_{\mathrm{ori}}-\bw_{\mathrm{un}}\rangle
}{
\|\Delta \bw_t\|_2 \, \|\bw_{\mathrm{ori}}-\bw_{\mathrm{un}}\|_2
}.
\label{eq:rollback_cos}
\end{equation}

A positive $\mathrm{RT}(t)$ indicates that continued training moves the model toward the pre-unlearning region, partially undoing the unlearning effect.

As shown in Fig.~\ref{fig:sim_step}, FedAU and FedOSD exhibit positive RT values, indicating a stronger tendency to move back toward the original model during continued training. 
This suggests that retained-data updates can pull these models toward the pre-unlearning region, where the removed target knowledge may be recovered. 
In contrast, \textsc{Lethe} keeps RT close to zero throughout continued training, indicating weak rollback tendency. 
This supports the role of Phase~2 in Algorithm~\ref{alg:Lethe}: the dual-stream rectification redirects the updates away from the correlated region, producing an unlearned model that is less likely to be pulled back by later retained-data training. 
NoT also suppresses rollback, but its unlearning-set loss remains low, indicating insufficient knowledge removal.

\begin{figure}[t]
\centering
\includegraphics[width=\linewidth]{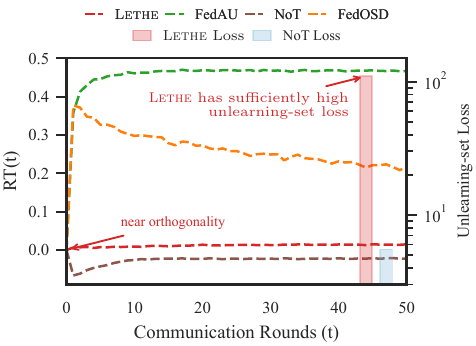}
\caption{Rollback tendency (RT) during continued training. 
Positive RT values indicate that training moves the model toward the original model. \textsc{Lethe} remains near-orthogonal to the rollback direction, whereas FedAU and FedOSD show stronger rollback tendency.}
\label{fig:sim_step}
\end{figure}

\textbf{Linear mode connectivity analysis.}
\label{para:lmc_analysis}
We further perform linear mode connectivity (LMC) analysis~\cite{frankle2020linear}. 
Recent work~\cite{siddiqui2025dormant} suggests that a model more resistant to knowledge resurfacing should exhibit a high-loss barrier between the unlearned model and a model where the removed knowledge is recovered. 
A smooth low-loss path indicates that the removed knowledge can be easily recovered, whereas a high-loss barrier suggests stronger resistance to knowledge resurfacing. 
In our setting, we analyze the interpolation path between the unlearned model $w_{\mathrm{un}}$ and a recovery model $w_{\mathrm{rec}}$, where $w_{\mathrm{rec}}$ is obtained by further training on the full data until convergence. 
As shown in Fig.~\ref{fig:lmc}, only \textsc{Lethe} exhibits a clear high-loss barrier along this path, while the other methods transition smoothly with low loss. 
This explains why \textsc{Lethe} is more resistant to knowledge resurfacing: the removed knowledge is harder to recover through a smooth optimization path.

\begin{figure*}[t]
    \centering
    \includegraphics[width=0.95\textwidth]{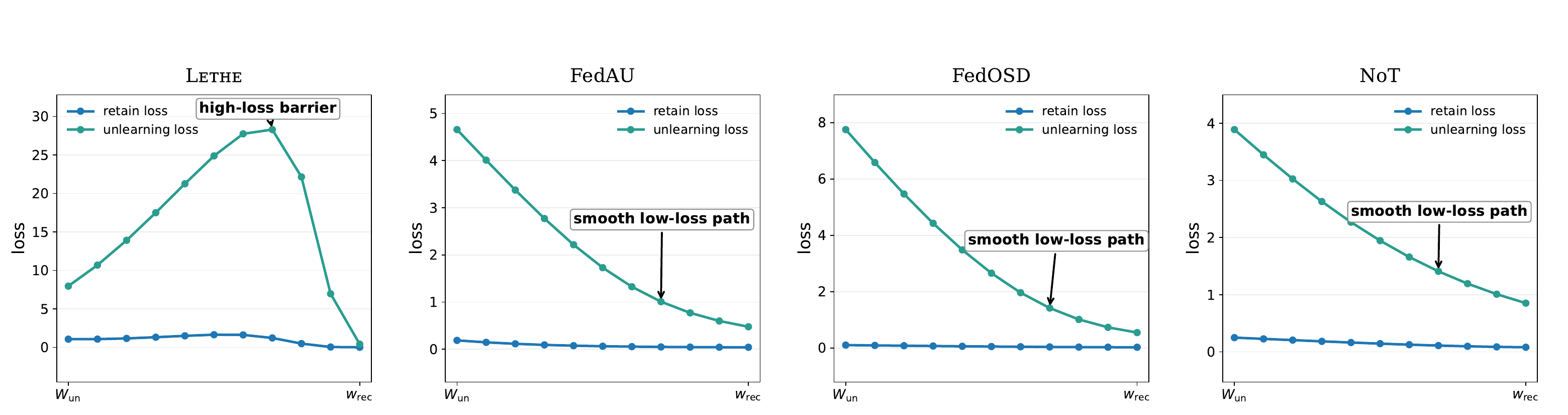}
    \caption{Recovery-path LMC between the unlearned model $w_{\mathrm{un}}$ and the recovery model $w_{\mathrm{rec}}$. 
    Existing methods exhibit smooth low-loss paths on the unlearning set, indicating that the removed knowledge can be easily recovered. 
    In contrast, \textsc{Lethe} forms a pronounced high-loss barrier, suggesting stronger resistance to knowledge resurfacing.}
    \label{fig:lmc}
\end{figure*}

\subsection{Long-horizon continued training}
We further conduct a long-horizon continued-training experiment to evaluate the persistence of 
\textsc{Lethe} under extended federated training. Specifically, we continue training for 50 rounds 
on TOFU with Llama-3.2-3B-Instruct and 1000 rounds on CIFAR-10 with ResNet-18, by which point both
 models have largely converged. As shown in Table~\ref{tab:long_horizon}, \textsc{Lethe} maintains
  KR close to the Retrain baseline in both settings. On TOFU, \textsc{Lethe} yields slightly higher
   KR than Retrain but noticeably higher MU, suggesting that the small KR gap may partly reflect
    stronger semantic generalization rather than direct recovery of memorized target samples.
     These results support the persistence of \textsc{Lethe} over the evaluated horizons.

\subsection{Ablation Study and Parameter Sensitivity}
To assess the roles of the temporary adapter, dual-stream
rectification, layer-wise conditionality, and adapter removal
in the \textsc{Lethe} amplify–rectify–restore pipeline, we evaluate
four variants using the same federation and data split as the
complete method, with results reported in Appendix S4-A. Besides, our method has only one hyper-parameter $\gamma$ employed in the rectification operation. We study its impact on both CV and NLP tasks, with results presented in Appendices S4-B and S4-C.

\section{Limitations and Discussion}

One limitation of \textsc{Lethe} is that it requires the unlearning client to participate during the unlearning stage to provide the forget stream. 
This assumption may not hold if the client completely leaves the federation and becomes unreachable. 
However, similar assumptions are also adopted in existing FU methods, such as FedAU and FedOSD, where the unlearning client participates in the unlearning procedure to provide information or updates~\cite{fedau,pan2025fedosd}. 
Moreover, this setting is reasonable in practical collaborations where a client wants to remove outdated or sensitive data but does not intend to quit the federation. 
For example, a subsidiary company may request data removal while still benefiting from the shared global model. 
In such cases, the client is naturally willing to cooperate in the unlearning process, and the remaining clients can continue federated training afterward.

\begin{table}[t]
\centering
\caption{Long-horizon continued training after unlearning. We run continued training 
until convergence and compare \textsc{Lethe} with the clean retraining baseline. 
For CV tasks, we train for \textbf{1000 rounds}; for NLP tasks,
 we train for 50 rounds.}
\label{tab:long_horizon}
\tiny
\setlength{\tabcolsep}{3pt}
\renewcommand{\arraystretch}{0.9}
\resizebox{0.75\columnwidth}{!}{
\begin{tabular}{lcccc}
\toprule
\multirow{2}{*}{\textbf{Method}} 
& \multicolumn{2}{c}{\textbf{CV-tasks}} 
& \multicolumn{2}{c}{\textbf{NLP-tasks}} \\
\cmidrule(lr){2-3} \cmidrule(lr){4-5}
& \textbf{MU} ($\uparrow$) 
& \textbf{KR} ($\downarrow$) 
& \textbf{MU} ($\uparrow$) 
& \textbf{KR} ($\downarrow$) \\
\midrule
Retrain 
& 62.80 
& 3.11
& 0.5695 
& {0.4046} \\
\textsc{Lethe} 
& {62.94} 
& {2.22} 
& {0.6122} 
& 0.4167 \\
\bottomrule
\end{tabular}
}
\vspace{-2mm}
\end{table}

\section{Conclusion}
\label{sec:conclusion}
We study FU in a deployment-realistic setting where deletion requests arrive amid long-lasting federated training. 
The experiments characterize \emph{knowledge resurfacing} as a key failure mode in this setting: target knowledge removed by unlearning can resurface within only a few additional communication rounds. 
Beyond existing measures of unlearning effectiveness, \emph{Resurfacing Rate} characterizes the degradation of the unlearning effect during continued training. 
Extensive experiments across diverse models, datasets, and unlearning settings show that many state-of-the-art methods are prone to knowledge resurfacing, whereas \textsc{Lethe} maintains a long-lasting unlearning effect throughout follow-up training. We interpret the principle of our method with theoretical analysis and provide a conditional finite-horizon bound on forget-loss recovery, which helps to understand why it achieves persistent knowledge erasure.

\textsc{Lethe} requires the participation of the unlearning client to provide the forget stream and relies on retained clients to help restore model utility. 
Reducing such participation requirements and developing more autonomous calibration of the rectification strength with stronger theoretical guarantees are important directions for our future work.

\section*{Acknowledgment}
This work is supported by the New Generation Artificial Intelligence-National Science and Technology Major Project (2025ZD0123605, 2025ZD0123604), National Natural Science Foundation of China (62402198), Basic and Applied Basic Research Project of Guangzhou (2025A04J2212), the Fundamental Research Funds for the Central Universities (21624348).

An earlier preprint version of this article entitled
 ``Lethe: Adapter-Augmented Dual-Stream Update for Persistent Knowledge Erasure in Federated
  Unlearning'' was posted on arXiv (arXiv:2601.22601).
     
\bibliographystyle{IEEEtran}
\bibliography{allreference}

\clearpage
\begingroup
\renewcommand{\encodingdefault}{T1}
\fontencoding{T1}\selectfont
\raggedbottom
\setlength{\textfloatsep}{8pt plus 1pt minus 2pt}
\setlength{\floatsep}{7pt plus 1pt minus 2pt}
\setlength{\intextsep}{8pt plus 1pt minus 2pt}

\setcounter{section}{0}
\setcounter{subsection}{0}
\setcounter{subsubsection}{0}
\setcounter{equation}{0}
\setcounter{figure}{0}
\setcounter{table}{0}
\renewcommand{\thesection}{S\arabic{section}}
\renewcommand{\thesubsection}{S\arabic{section}-\Alph{subsection}}
\renewcommand{\thesubsubsection}{S\arabic{section}-\Alph{subsection}.\arabic{subsubsection}}
\renewcommand{\theequation}{S\arabic{equation}}
\renewcommand{\thefigure}{S\arabic{figure}}
\renewcommand{\thetable}{S\arabic{table}}

\renewcommand{\theHsection}{supp.\arabic{section}}
\renewcommand{\theHsubsection}{supp.\arabic{section}.\arabic{subsection}}
\renewcommand{\theHsubsubsection}{supp.\arabic{section}.\arabic{subsection}.\arabic{subsubsection}}
\renewcommand{\theHequation}{supp.\arabic{equation}}
\renewcommand{\theHfigure}{supp.\arabic{figure}}
\renewcommand{\theHtable}{supp.\arabic{table}}

\markboth{Preprint --- Supplementary Material}%
{Wu \MakeLowercase{\textit{et al.}}: Supplementary Material for \textsc{Lethe}}
\twocolumn[{
  \begin{center}
    {\Large\bfseries Supplementary Material\par}
  \end{center}
  \vspace{0.5em}
}]
\pdfbookmark[0]{Supplementary Material}{supplementary-material}
\setlength{\abovedisplayskip}{5pt plus 2pt minus 2pt}
\setlength{\belowdisplayskip}{5pt plus 2pt minus 2pt}
\setlength{\abovedisplayshortskip}{2pt plus 2pt}
\setlength{\belowdisplayshortskip}{4pt plus 2pt minus 2pt}

\section{Detailed Proofs of the Theoretical Results}
\label{app:theoretical-proofs}
This section provides complete derivations for the theoretical results in Sec.~IV-D of the main paper. We follow its notation: \(\bar{\bw}_0=\bw_{\mathrm{un}}\), \(\boldsymbol r_h=\mathcal{T}_r(\bar{\bw}_h)-\bar{\bw}_h\), \(\boldsymbol u_h=-\eta_u\nabla\mathcal{L}_u(\bar{\bw}_h)\), and \(a_h=\langle\boldsymbol r_h,\boldsymbol u_h\rangle\). We also write \([x]_+=\max\{x,0\}\) and \(\mathcal{B}_R(\bw)=\{\bw':\|\bw'-\bw\|\le R\}\). Here, \(\mathsf{LR}_H=\max_{0\le j\le H}[\mathcal{L}_u(\bar{\bw}_0)-\mathcal{L}_u(\bar{\bw}_j)]_+\) denotes the largest decrease in unlearning loss over \(H\) retained-data continued-training rounds.

The derivations follow two complementary chains. The persistence chain propagates a terminal negative recovery-alignment margin over a local neighborhood and combines it with a finite path budget to derive a multi-round upper bound on \(\mathsf{LR}_H\). The rollback chain tracks distance to the pre-unlearning model and derives a recovery lower bound under contraction in the old retained-data attraction basin. Proposition~IV.1 first establishes the layer-wise rectification identities used in these analyses.

\subsection{Proof of Proposition IV.1}
\label{app:proof-rectification}
We verify the two branches of the layer-wise rectification rule separately.
\begin{proof}
Fix a layer \(\ell\) and abbreviate
\(s=s_t^{(\ell)}
=\langle\Delta_t^{r,(\ell)},\Delta_t^{u,(\ell)}\rangle\).
When \(s\le0\), the rectification rule gives
\(\widetilde{\Delta}_t^{(\ell)}=-\Delta_t^{u,(\ell)}\). Therefore,
\begin{align}
\left\langle\widetilde{\Delta}_t^{(\ell)},
\Delta_t^{u,(\ell)}\right\rangle
&=\left\langle-\Delta_t^{u,(\ell)},
\Delta_t^{u,(\ell)}\right\rangle \nonumber\\
&=-\left\|\Delta_t^{u,(\ell)}\right\|^2,
\label{eq:supp_rectify_forget}
\\
\left\langle\widetilde{\Delta}_t^{(\ell)},
\Delta_t^{r,(\ell)}\right\rangle
&=-\left\langle\Delta_t^{u,(\ell)},
\Delta_t^{r,(\ell)}\right\rangle \nonumber\\
&=-s\ge0.
\label{eq:supp_rectify_retain}
\end{align}
Thus, this branch reverses the forget-stream update while remaining
non-negatively aligned with the retain stream.

When \(s>0\), the rectification rule gives
\(\widetilde{\Delta}_t^{(\ell)}
=\Delta_t^{r,(\ell)}-\gamma\Delta_t^{u,(\ell)}\). Direct subtraction yields
\begin{equation}
\widetilde{\Delta}_t^{(\ell)}-\Delta_t^{r,(\ell)}
=-\gamma\Delta_t^{u,(\ell)}.
\label{eq:supp_rectify_difference}
\end{equation}
Taking its inner product with the forget-stream update gives
\begin{align}
\left\langle\widetilde{\Delta}_t^{(\ell)},
\Delta_t^{u,(\ell)}\right\rangle
&=\left\langle\Delta_t^{r,(\ell)}
-\gamma\Delta_t^{u,(\ell)},
\Delta_t^{u,(\ell)}\right\rangle \nonumber\\
&=s-\gamma\left\|\Delta_t^{u,(\ell)}\right\|^2.
\label{eq:supp_rectify_positive}
\end{align}
The condition \(s>0\) implies
\(\Delta_t^{u,(\ell)}\ne\boldsymbol 0\). Hence, for every \(\gamma>0\),
\(\gamma\|\Delta_t^{u,(\ell)}\|^2>0\), and the last expression is strictly
smaller than \(s\).

For the first-order loss interpretation, set the stream-approximation errors
defined in Sec.~IV-D of the main paper to zero and write
\(\boldsymbol g_u=\nabla\mathcal{L}_u(\bw_t)\) and
\(\boldsymbol g_r=\nabla\mathcal{L}_r(\bw_t)\), so that
\(\Delta_t^u=-\eta_u\boldsymbol g_u\) and
\(\Delta_t^r=-\eta_r\boldsymbol g_r\). In the \(s\le0\) branch,
\begin{align}
\left\langle\boldsymbol g_u,\widetilde{\Delta}_t\right\rangle
&=\eta_u\|\boldsymbol g_u\|^2\ge0,
\label{eq:supp_rectify_forget_directional}\\
\left\langle\boldsymbol g_r,\widetilde{\Delta}_t\right\rangle
&=-\frac{1}{\eta_r}
\left\langle\Delta_t^r,\widetilde{\Delta}_t\right\rangle
=\frac{s}{\eta_r}\le0.
\label{eq:supp_rectify_retain_directional}
\end{align}
Thus, the applied update increases the deleted-data loss and does not increase
the retained-data loss to first order. In the \(s>0\) branch, the change made
relative to the unrectified retain update is
\begin{align}
\left\langle\boldsymbol g_u,
\widetilde{\Delta}_t-\Delta_t^r\right\rangle
&=\left\langle\boldsymbol g_u,-\gamma\Delta_t^u\right\rangle
\nonumber\\
&=\gamma\eta_u\|\boldsymbol g_u\|^2>0.
\label{eq:supp_rectify_relative_forgetting}
\end{align}
Consequently, this branch adds a strict reverse update on the deleted
objective relative to ordinary retained-data training, while
Eq.~\eqref{eq:supp_rectify_positive} gives the corresponding reduction in
forget-stream alignment.
\end{proof}

\subsection{Proof of Proposition IV.2}
\label{app:proof-local-margin}
We next show how a negative recovery-alignment margin at the returned model
extends to a neighborhood of that model.
\begin{proof}
Let \(\boldsymbol r_0=\boldsymbol r(\bar{\bw}_0)\),
\(\boldsymbol u_0=\boldsymbol u(\bar{\bw}_0)\), and define
\begin{equation}
\delta\boldsymbol r=\boldsymbol r(\bw)-\boldsymbol r_0,\qquad
\delta\boldsymbol u=\boldsymbol u(\bw)-\boldsymbol u_0.
\end{equation}
Expanding the inner product gives
\begin{align}
a(\bw)-a(\bar{\bw}_0)
&=\langle\boldsymbol r_0+\delta\boldsymbol r,
\boldsymbol u_0+\delta\boldsymbol u\rangle
-\langle\boldsymbol r_0,\boldsymbol u_0\rangle \nonumber\\
&=\langle\delta\boldsymbol r,\boldsymbol u_0\rangle
+\langle\boldsymbol r_0,\delta\boldsymbol u\rangle
+\langle\delta\boldsymbol r,\delta\boldsymbol u\rangle.
\label{eq:supp_alignment_expand}
\end{align}
By Cauchy--Schwarz,
\begin{align}
a(\bw)-a(\bar{\bw}_0)
&\le
\|\delta\boldsymbol r\|\,\|\boldsymbol u_0\|
+\|\boldsymbol r_0\|\,\|\delta\boldsymbol u\|
+\|\delta\boldsymbol r\|\,\|\delta\boldsymbol u\|.
\label{eq:supp_alignment_cs}
\end{align}
For \(d=\|\bw-\bar{\bw}_0\|\), the Lipschitz assumptions imply
\(\|\delta\boldsymbol r\|\le K_rd\) and
\(\|\delta\boldsymbol u\|\le K_ud\). Substitution into
Eq.~\eqref{eq:supp_alignment_cs} yields
\begin{equation}
a(\bw)-a(\bar{\bw}_0)
\le
(K_r\|\boldsymbol u_0\|+K_u\|\boldsymbol r_0\|)d
+K_rK_ud^2.
\label{eq:supp_alignment_bound}
\end{equation}
For every \(\bw\in\mathcal{B}_R(\bar{\bw}_0)\), \(d\le R\), so the right-hand
side is at most \(C_R\). This step uses the fact that the polynomial
\begin{equation}
\begin{aligned}
C(d)&=Ad+Bd^2,\\
A&=K_r\|\boldsymbol u_0\|+K_u\|\boldsymbol r_0\|,\\
B&=K_rK_u.
\end{aligned}
\label{eq:supp_alignment_polynomial}
\end{equation}
is nondecreasing for \(d\ge0\), because \(A,B\ge0\). Combining the resulting
uniform bound with \(a(\bar{\bw}_0)\le-m_a\) gives
\begin{equation}
a(\bw)\le -m_a+C_R<0,
\end{equation}
where the strict inequality follows from \(C_R<m_a\).

The condition also gives an explicit radius. When \(B>0\), the
positive solution of \(AR+BR^2=m_a\) is
\begin{equation}
R_{\star}
=\frac{\sqrt{A^2+4Bm_a}-A}{2B}.
\label{eq:supp_certified_radius}
\end{equation}
Every \(0<R<R_{\star}\) that remains inside the assumed Lipschitz domain
satisfies \(C_R<m_a\). If \(B=0\) and \(A>0\), the corresponding condition is
\(R<m_a/A\); if \(A=B=0\), the alignment is constant on the domain. Hence,
each admissible \(R\) defines a neighborhood on which \(a(\bw)<0\).

When both streams are nonzero, the normalized terminal margin
\begin{equation}
\widehat m_a
=-\frac{a(\bar{\bw}_0)}
{\|\boldsymbol r_0\|\,\|\boldsymbol u_0\|}
\label{eq:supp_normalized_margin}
\end{equation}
is positive exactly when the terminal alignment is negative and gives the
normalized form of the premise in Proposition~IV.2.
\end{proof}

\subsection{Proof of Proposition IV.3}
\label{app:proof-rollback}
We derive the exact change in distance to the pre-unlearning model for one
continued-training round.
\begin{proof}
Because
\(\bar{\bw}_{h+1}=\bar{\bw}_h+\boldsymbol r_h\) and
\(\boldsymbol d_h=\bw_{\mathrm{ori}}-\bar{\bw}_h\),
\begin{equation}
\boldsymbol d_{h+1}
=\bw_{\mathrm{ori}}-\bar{\bw}_{h+1}
=\boldsymbol d_h-\boldsymbol r_h.
\end{equation}
Consequently,
\begin{align}
D_{h+1}^2
&=\|\boldsymbol d_h-\boldsymbol r_h\|^2 \nonumber\\
&=\|\boldsymbol d_h\|^2+\|\boldsymbol r_h\|^2
-2\langle\boldsymbol r_h,\boldsymbol d_h\rangle \nonumber\\
&=D_h^2+\|\boldsymbol r_h\|^2-2b_h.
\label{eq:supp_rollback_identity}
\end{align}
If \(b_h>\|\boldsymbol r_h\|^2/2\), then
\(D_{h+1}^2<D_h^2\). Since both distances are nonnegative, this implies
\(D_{h+1}<D_h\).

The threshold can also be written in angular form. If
\(D_h\|\boldsymbol r_h\|>0\), let \(\theta_h\) be the angle between
\(\boldsymbol r_h\) and \(\boldsymbol d_h\). Then
\begin{equation}
b_h
=\|\boldsymbol r_h\|D_h\cos\theta_h,
\end{equation}
and the rollback condition is equivalent to
\begin{equation}
\cos\theta_h>
\frac{\|\boldsymbol r_h\|}{2D_h}.
\label{eq:supp_rollback_cosine_threshold}
\end{equation}
Thus, the rollback threshold depends jointly on update direction, update
magnitude, and the current distance.

Summing the exact one-round identity from \(h=0\) to \(H-1\) additionally
gives the trajectory-level relation
\begin{equation}
D_H^2-D_0^2
=\sum_{h=0}^{H-1}\|\boldsymbol r_h\|^2
-2\sum_{h=0}^{H-1}b_h.
\label{eq:supp_cumulative_rollback}
\end{equation}
Hence, the trajectory finishes closer to the pre-unlearning model whenever
the accumulated rollback alignment exceeds one half of the accumulated
squared update norm.
\end{proof}

\subsection{Proof of Proposition IV.4}
\label{app:proof-old-basin}
We now turn contraction within the attraction basin defined in
Proposition~IV.4 into a lower bound on eventual forget-loss recovery.
\begin{proof}
By definition of \(\mathcal{A}_{\mathrm{old}}\), applying
\(\mathcal{T}_r\) to any point in this basin produces another point whose
iterates converge to \(\bw_r^{\mathrm{old}}\). More explicitly, if
\(\bw\in\mathcal{A}_{\mathrm{old}}\), then
\begin{equation}
\lim_{j\to\infty}\mathcal{T}_r^j(\mathcal{T}_r(\bw))
=\lim_{j\to\infty}\mathcal{T}_r^{j+1}(\bw)
=\bw_r^{\mathrm{old}}.
\label{eq:supp_forward_invariance}
\end{equation}
Therefore, \(\mathcal{T}_r(\bw)\in\mathcal{A}_{\mathrm{old}}\), and the basin
is forward invariant. Starting from
\(\bar{\bw}_0\in\mathcal{A}_{\mathrm{old}}\), all iterates remain in the basin,
so the contraction assumption can be applied at every round.

We next prove the finite-round contraction bound by induction. It holds at
\(h=0\) with equality. If it holds for some \(h\), then
\begin{align}
\|\bar{\bw}_{h+1}-\bw_r^{\mathrm{old}}\|
&=\|\mathcal{T}_r(\bar{\bw}_h)-\bw_r^{\mathrm{old}}\| \nonumber\\
&\le\kappa\|\bar{\bw}_h-\bw_r^{\mathrm{old}}\| \nonumber\\
&\le\kappa^{h+1}
\|\bar{\bw}_0-\bw_r^{\mathrm{old}}\|.
\end{align}
Thus, for every \(h\ge0\),
\begin{equation}
\|\bar{\bw}_h-\bw_r^{\mathrm{old}}\|
\le
\kappa^h\|\bar{\bw}_0-\bw_r^{\mathrm{old}}\|.
\label{eq:supp_old_basin_contraction}
\end{equation}
Applying the triangle inequality and
\(\|\bw_r^{\mathrm{old}}-\bw_{\mathrm{ori}}\|\le\delta\) yields
\begin{align}
\|\bar{\bw}_h-\bw_{\mathrm{ori}}\|
&\le
\|\bar{\bw}_h-\bw_r^{\mathrm{old}}\|
+\|\bw_r^{\mathrm{old}}-\bw_{\mathrm{ori}}\| \nonumber\\
&\le
\kappa^h\|\bar{\bw}_0-\bw_r^{\mathrm{old}}\|+\delta.
\label{eq:supp_old_basin_distance}
\end{align}
Lipschitz continuity of \(\mathcal{L}_u\) gives
\begin{align}
\mathcal{L}_u(\bw_r^{\mathrm{old}})
&\le
\mathcal{L}_u(\bw_{\mathrm{ori}})
+G_{\mathrm{Lip}}
\|\bw_r^{\mathrm{old}}-\bw_{\mathrm{ori}}\| \nonumber\\
&\le
\mathcal{L}_u(\bw_{\mathrm{ori}})
+G_{\mathrm{Lip}}\delta.
\label{eq:supp_old_basin_loss}
\end{align}
Because \(0<\kappa<1\), Eq.~\eqref{eq:supp_old_basin_contraction} also proves
\(\bar{\bw}_h\to\bw_r^{\mathrm{old}}\). Continuity of \(\mathcal{L}_u\), which
is implied by its local Lipschitz property, then gives
\(\mathcal{L}_u(\bar{\bw}_h)\to
\mathcal{L}_u(\bw_r^{\mathrm{old}})\).

For every \(H\), the maximization in the definition of \(\mathsf{LR}_H\)
contains the choice \(j=H\). Hence,
\begin{equation}
\mathsf{LR}_H
\ge
[\mathcal{L}_u(\bar{\bw}_0)-\mathcal{L}_u(\bar{\bw}_H)]_+.
\label{eq:supp_old_basin_lr_finite}
\end{equation}
The positive-part mapping is continuous. Taking \(H\to\infty\) in the
right-hand side of Eq.~\eqref{eq:supp_old_basin_lr_finite}, and then applying
Eq.~\eqref{eq:supp_old_basin_loss}, yields
\begin{align}
\liminf_{H\to\infty}\mathsf{LR}_H
&\ge
[\mathcal{L}_u(\bar{\bw}_0)
-\mathcal{L}_u(\bw_r^{\mathrm{old}})]_+ \nonumber\\
&\ge
[\mathcal{L}_u(\bar{\bw}_0)
-\mathcal{L}_u(\bw_{\mathrm{ori}})
-G_{\mathrm{Lip}}\delta]_+.
\end{align}
Writing
\(m_u=\mathcal{L}_u(\bar{\bw}_0)-
\mathcal{L}_u(\bw_{\mathrm{ori}})\), the last expression is strictly positive
whenever \(m_u>G_{\mathrm{Lip}}\delta\), which proves the claimed
forget-loss-recovery lower bound.
\end{proof}

\subsection{Proof of Theorem IV.1}
\label{app:proof-forget-loss-recovery-bound}
We derive a per-round loss-recovery inequality and then sum it along the
continued-training trajectory.
\begin{proof}
For \(\bw\) and displacement \(\boldsymbol v\) whose
connecting segment lies in the assumed smoothness neighborhood,
the fundamental theorem of calculus along the segment
\(\bw+\tau\boldsymbol v\), \(\tau\in[0,1]\), gives
\begin{align}
\mathcal{L}_u(\bw+\boldsymbol v)-\mathcal{L}_u(\bw)
&=\int_0^1
\left\langle\nabla\mathcal{L}_u(\bw+\tau\boldsymbol v),
\boldsymbol v\right\rangle d\tau \nonumber\\
&=\left\langle\nabla\mathcal{L}_u(\bw),\boldsymbol v\right\rangle
+\mathcal{R}(\bw,\boldsymbol v),
\label{eq:supp_smooth_integral}
\end{align}
where
\begin{align}
|\mathcal{R}(\bw,\boldsymbol v)|
&\le\int_0^1
\|\nabla\mathcal{L}_u(\bw+\tau\boldsymbol v)
-\nabla\mathcal{L}_u(\bw)\|\,\|\boldsymbol v\|d\tau \nonumber\\
&\le\int_0^1\beta_u\tau\|\boldsymbol v\|^2d\tau
=\frac{\beta_u}{2}\|\boldsymbol v\|^2.
\label{eq:supp_smooth_remainder}
\end{align}
Taking the lower sign of this two-sided remainder bound yields
\begin{equation}
\mathcal{L}_u(\bw+\boldsymbol v)
\ge
\mathcal{L}_u(\bw)
+\langle\nabla\mathcal{L}_u(\bw),\boldsymbol v\rangle
-\frac{\beta_u}{2}\|\boldsymbol v\|^2.
\label{eq:supp_lower_taylor}
\end{equation}
Set \(\bw=\bar{\bw}_h\) and \(\boldsymbol v=\boldsymbol r_h\). Since
\(\bar{\bw}_{h+1}=\bar{\bw}_h+\boldsymbol r_h\), rearranging
Eq.~\eqref{eq:supp_lower_taylor} gives
\begin{align}
\mathcal{L}_u(\bar{\bw}_h)-\mathcal{L}_u(\bar{\bw}_{h+1})
&\le
-\langle\nabla\mathcal{L}_u(\bar{\bw}_h),\boldsymbol r_h\rangle
+\frac{\beta_u}{2}\|\boldsymbol r_h\|^2 \nonumber\\
&=
\frac{a_h}{\eta_u}
+\frac{\beta_u}{2}\|\boldsymbol r_h\|^2 \nonumber\\
&\le
\frac{[a_h]_+}{\eta_u}
+\frac{\beta_u}{2}\|\boldsymbol r_h\|^2,
\label{eq:supp_per_round_bound}
\end{align}
where the equality uses
\(\boldsymbol u_h=-\eta_u\nabla\mathcal{L}_u(\bar{\bw}_h)\), because
\begin{equation}
a_h
=-\eta_u
\left\langle\boldsymbol r_h,
\nabla\mathcal{L}_u(\bar{\bw}_h)\right\rangle.
\label{eq:supp_alignment_gradient_relation}
\end{equation}
The last inequality in Eq.~\eqref{eq:supp_per_round_bound} follows from the
scalar relation \(a_h\le[a_h]_+\).

For any \(j\le H\), the exact telescoping identity
\begin{equation}
\mathcal{L}_u(\bar{\bw}_0)-\mathcal{L}_u(\bar{\bw}_j)
=\sum_{h=0}^{j-1}
\bigl(\mathcal{L}_u(\bar{\bw}_h)
-\mathcal{L}_u(\bar{\bw}_{h+1})\bigr)
\label{eq:supp_loss_telescoping}
\end{equation}
and Eq.~\eqref{eq:supp_per_round_bound} yield
\begin{align}
\mathcal{L}_u(\bar{\bw}_0)-\mathcal{L}_u(\bar{\bw}_j)
&\le
\sum_{h=0}^{j-1}
\left(
\frac{[a_h]_+}{\eta_u}
+\frac{\beta_u}{2}\|\boldsymbol r_h\|^2
\right) \nonumber\\
&\le
\sum_{h=0}^{H-1}
\left(
\frac{[a_h]_+}{\eta_u}
+\frac{\beta_u}{2}\|\boldsymbol r_h\|^2
\right).
\end{align}
The final sum is nonnegative. Taking the positive part of the left-hand side
and maximizing over \(0\le j\le H\) proves the theorem.

For the zero-recovery condition stated in the main paper, suppose that, for
every $0\le h<H$,
\begin{equation}
a_h\le-\frac{\beta_u\eta_u}{2}\|\boldsymbol r_h\|^2,
\label{eq:supp_strong_negative_alignment}
\end{equation}
then the middle expression in Eq.~\eqref{eq:supp_per_round_bound}, before
replacing \(a_h\) by \([a_h]_+\), is nonpositive. Therefore,
\(\mathcal{L}_u(\bar{\bw}_{h+1})\ge
\mathcal{L}_u(\bar{\bw}_h)\) throughout these $H$ rounds, and
\(\mathsf{LR}_H=0\) for the specified horizon.
\end{proof}

\subsection{Proof of Corollary IV.1}
\label{app:proof-finite-horizon}
We fix an integer $H\ge1$ and combine local alignment stability with update
bounds $\rho_h\ge0$ for $0\le h<H$.
\begin{proof}
For every $1\le j\le H$, the trajectory identity and triangle inequality give
\begin{align}
\|\bar{\bw}_j-\bar{\bw}_0\|
&=
\left\|\sum_{h=0}^{j-1}\boldsymbol r_h\right\|
\le
\sum_{h=0}^{j-1}\|\boldsymbol r_h\| \nonumber\\
&\le
\sum_{h=0}^{j-1}\rho_h
\le\sum_{h=0}^{H-1}\rho_h\le R.
\label{eq:supp_path_length}
\end{align}
The case $j=0$ is immediate. Thus all $H+1$ iterates, and the segments joining
them by convexity of the ball, lie in $\mathcal{B}_R(\bar{\bw}_0)$.
Proposition~IV.2 gives, for each $0\le h<H$,
\begin{equation}
a_h\le a(\bar{\bw}_0)+C_R
\le-m_a+C_R=-\varepsilon_a<0.
\label{eq:supp_trajectory_negative_alignment}
\end{equation}
Retaining $a_h$ before taking its positive part in
Eq.~\eqref{eq:supp_per_round_bound}, we obtain
\begin{align}
\mathcal{L}_u(\bar{\bw}_h)-\mathcal{L}_u(\bar{\bw}_{h+1})
&\le
\frac{a_h}{\eta_u}+\frac{\beta_u}{2}\|\boldsymbol r_h\|^2 \nonumber\\
&\le
-\frac{\varepsilon_a}{\eta_u}+\frac{\beta_u}{2}\rho_h^2.
\label{eq:supp_margin_recovery_step}
\end{align}
Telescoping up to any $1\le j\le H$ yields
\begin{equation}
\mathcal{L}_u(\bar{\bw}_0)-\mathcal{L}_u(\bar{\bw}_j)
\le\frac{\beta_u}{2}\sum_{h=0}^{j-1}\rho_h^2
-\frac{j\varepsilon_a}{\eta_u}.
\label{eq:supp_margin_prefix_bound}
\end{equation}
Both sides are zero for $j=0$ under the empty-sum convention. Taking positive
parts and maximizing over $0\le j\le H$ proves
\begin{equation}
\begin{aligned}
\mathsf{LR}_H
&\le\max_{0\le j\le H}
\left[\frac{\beta_u}{2}\sum_{h=0}^{j-1}\rho_h^2
-\frac{j\varepsilon_a}{\eta_u}\right]_+\\
&\le\frac{\beta_u}{2}\sum_{h=0}^{H-1}\rho_h^2.
\end{aligned}
\label{eq:supp_finite_horizon_bound}
\end{equation}
The last inequality follows by dropping the nonpositive margin term and
extending each nonnegative sum to $H$. In particular, the first bound is zero
if $\beta_u\rho_h^2/2\le\varepsilon_a/\eta_u$ for every $0\le h<H$.
\end{proof}

\paragraph{Geometric specialization}
For $G_r>0$ and $q\in(0,1)$, the finite envelope $\rho_h=G_rq^h$ is admissible
when
\begin{equation}
\|\boldsymbol r_h\|\le G_rq^h\quad(0\le h<H),
\qquad\frac{G_r(1-q^H)}{1-q}\le R.
\label{eq:supp_finite_geometric_path}
\end{equation}
Using the finite geometric-series identity in
Eq.~\eqref{eq:supp_finite_horizon_bound} gives
\begin{equation}
\begin{aligned}
\mathsf{LR}_H
&\le\max_{0\le j\le H}
\left[\frac{\beta_uG_r^2(1-q^{2j})}{2(1-q^2)}
-\frac{j\varepsilon_a}{\eta_u}\right]_+\\
&\le\frac{\beta_uG_r^2(1-q^{2H})}{2(1-q^2)}.
\end{aligned}
\label{eq:supp_finite_geometric_bound}
\end{equation}
\paragraph{Example}
Consider scalar $w$, losses
$\mathcal{L}_r(w)=\tfrac12(w-1)^2$ and
$\mathcal{L}_u(w)=\tfrac12(w+1)^2$, one retained gradient step of size $0.01$,
and $\eta_u=1$. Starting at $\bar w_0=0$ gives
\begin{equation}
\begin{aligned}
r(w)&=0.01(1-w),& u(w)&=-(1+w),\\
\bar w_h&=1-0.99^h,& r_h&=0.01(0.99)^h.
\end{aligned}
\label{eq:supp_finite_example}
\end{equation}
On $\mathcal{B}_{0.1}(0)$, take $K_r=0.01$, $K_u=1$, $\beta_u=1$, and
$m_a=0.01=-a_0$. The local stability constant is
$C_R=(0.01+0.01)(0.1)+0.01(0.1)^2=0.0021$, so
$\varepsilon_a=0.0079>0$. With $H=10$, $G_r=0.01$, and $q=0.99$, the path
budget is $1-0.99^{10}\approx0.0956<0.1$. All premises therefore hold.
Moreover, $\beta_u\rho_h^2/2\le0.00005<\varepsilon_a/\eta_u$, so
Corollary~IV.1 gives $\mathsf{LR}_{10}=0$. Along this nonconstant trajectory,
the retained loss decreases while the unlearning loss increases.
\section{Experimental Details}
\label{sec:experimental-details}

\subsection{Federation and Training Configuration}

\textbf{CV tasks.}
For all CV experiments, we partition the training data across $K=50$ clients using a Dirichlet distribution with concentration parameter $\alpha=0.1$. All clients participate in every communication round. During unlearning, the rectification phase involves all participating clients, while the subsequent recovery phase involves only the retained clients. For LeNet-5 and ResNet-18, we use a batch size of $32$ and a learning rate of $0.001$, with SGD using momentum $0.9$ and no weight decay. For HSViT, we use a batch size of $512$ and AdamW with a learning rate of $0.001$, weight decay $0.001$, and $(\beta_1,\beta_2)=(0.9,0.999)$. Each client performs one local epoch per communication round in all CV experiments.

\textbf{NLP tasks.}
For the NLP experiments, we use $K=20$ clients with a $100\%$ participation rate in every communication round. The rectification phase includes all participating clients, whereas the recovery phase uses only the retained clients. For TOFU with Llama-3.2-3B-Instruct and LUME Task 2 with OLMo-1B, both models are fine-tuned using LoRA with rank $32$, scaling factor $64$, dropout $0.05$, a learning rate of $8\times10^{-5}$, and weight decay $0.01$. Each client performs five local epochs per communication round on TOFU and one local epoch per communication round on LUME Task 2.

\subsection{Randomness and Implementation}
All experiments are repeated over multiple independent runs, and the tabulated results report the mean across runs. All experiments are executed on an NVIDIA GeForce RTX 5090 D GPU.

For tabulated results, U-KR is evaluated immediately after unlearning, whereas CT-KR is evaluated at the final checkpoint $T$ of the pre-specified continued-training schedule. RR is computed directly from the endpoint KR values using the definition in the main paper; displayed KR values are rounded only after this calculation. RR is reported as NA when its normalization denominator is non-positive.


\section{Temporary Adapter Design and Cost}
\label{sec:adapter-training-cost}

\subsection{Placement and Parameterization}
For all CV backbones, the temporary adapter is inserted immediately before the
final fully connected (FC) classification layer. Given a penultimate feature
vector of dimension $d$, the adapter applies a bottleneck mapping
$d\rightarrow m\rightarrow d$; the resulting $d$-dimensional representation is
then passed to the original FC classifier. This placement leaves the backbone
and classifier dimensions unchanged and confines the temporary trainable
parameters to the two adapter projections. Including both projection biases,
the adapter contains
\begin{equation}
N_{\mathrm{adapter}}=dm+m+md+d
\end{equation}
trainable parameters. Table~\ref{tab:adapter_configuration} lists the
architecture-specific configurations.

\begin{table}[!ht]
    \centering
    \caption{Temporary adapter configurations for the CV backbones.}
    \label{tab:adapter_configuration}
    \small
    \setlength{\tabcolsep}{6pt}
    \begin{tabular}{lcc}
        \toprule
        \textbf{Model} & \textbf{Adapter structure} & \textbf{Parameters} \\
        \midrule
        LeNet-5   & $84\rightarrow42\rightarrow84$   & $7{,}182$ \\
        ResNet-18 & $512\rightarrow64\rightarrow512$ & $66{,}112$ \\
        HSViT     & $64\rightarrow32\rightarrow64$   & $4{,}192$ \\
        \bottomrule
    \end{tabular}
\end{table}

\subsection{Training and Lifecycle}

The adapter is optimized only during signal amplification while the backbone remains frozen. The adapter-step count $S_{\phi}$ is determined by the unlearning-set size $|\mathcal{D}_u|$ and adapter batch size $B_u$. Specifically, one traversal of the unlearning data gives

\begin{equation}
S_{\phi}=\left\lceil\frac{|\mathcal{D}_u|}{B_u}\right\rceil,
\label{eq:sphi_updates}
\end{equation}

mini-batch updates, and adapter optimization ends after this epoch. Before the
 first rectification round, the unlearning client uploads the trained adapter 
 once, and the server distributes it through one downlink broadcast. Participating
  clients cache the frozen adapter throughout dual-stream rectification. In every
   subsequent rectification round, the server transmits only the current backbone;
    the adapter is neither uploaded nor rebroadcast. The adapter is deleted after 
    rectification and before recovery; it therefore adds no adapter-related
     computation or storage during recovery or inference. The NLP task experiments
      do not use this temporary adapter and apply dual-stream rectification
       directly to the LoRA-based local updates.

\subsection{Communication and Computation Accounting}
The communication and FLOPs results in the main paper use end-to-end accounting
and include the temporary adapter overhead. Communication includes its single
client-to-server upload and single server-to-client broadcast; neither is
repeated in later rectification rounds. FLOPs include
both adapter optimization during signal amplification and the frozen adapter's
forward computation during rectification. Table~\ref{tab:adapter_communication}
summarizes the one-time communication payload.

\begin{table}[!ht]
    \centering
    \caption{One-time communication overhead of the temporary adapters.}
    \label{tab:adapter_communication}
    \scriptsize
    \setlength{\tabcolsep}{3pt}
    \begin{tabular}{lcc}
        \toprule
        \textbf{Model} & $N_{\phi}$ & \textbf{Total payload} \\
        \midrule
        LeNet-5   & $7{,}182$   & $14{,}364$ parameters \\
        ResNet-18 & $66{,}112$  & $132{,}224$ parameters \\
        HSViT     & $4{,}192$   & $8{,}384$ parameters \\
        \bottomrule
    \end{tabular}
\end{table}

The total payload is $2N_{\phi}$ parameters per adapter: one upload and one
logical downlink broadcast. If physical traffic is charged separately for each
recipient, the downlink payload scales with the number of participating clients,
while the protocol still performs one broadcast event. The cached adapter is
not retransmitted in later rectification rounds and is absent during recovery
and deployment. The main-paper communication and FLOPs values are therefore
full-pipeline totals rather than backbone-only estimates.

\section{Ablation and Sensitivity Analysis}
\label{sec:ablation-sensitivity}

\subsection{Temporary Adapter Ablation}
\label{sec:adapter-ablation}
To assess the roles of the temporary adapter, dual-stream rectification,
layer-wise conditionality, and adapter removal in the \method{}
\emph{Amplify--Rectify--Restore} pipeline, we evaluate four variants using the
same federation and data split as the complete method. We report total
communication rounds by jointly counting the \emph{Rectify} (unlearning) and
\emph{Restore} (utility recovery) rounds. RR is computed at the final
continued-training checkpoint. All results in Table~\ref{tab:component_ablation}
use the main CV setting with $K=50$ clients.

\noindent\textbf{(I) w/o Adapter (Backbone-only Forget Stream).}
We omit the temporary adapter and construct the forget stream $\Delta_t^u$
directly on $\mathcal{D}_u$ by updating the backbone with the same Phase~2 local-training
hyperparameters as the complete method. This variant evaluates \method{}
without adapter-based signal amplification.

\noindent\textbf{(II) w/o Rectify (No-Correction Baseline).}
We retain \emph{Amplify} and construct both streams as in the complete method,
but bypass \emph{Rectify} by applying the retain stream unchanged,
$\widetilde{\Delta}_t^{(\ell)}\leftarrow\Delta_t^{r,(\ell)}$, at every layer.
This variant isolates the contribution of dual-stream rectification to
suppressing knowledge resurfacing during subsequent retained-data continued
training.

\noindent\textbf{(III) w/o Conditionality (Unconditional Rectify).}
We retain both streams and the rectification penalty but remove the layer-wise
alignment condition. Specifically, we apply
$\widetilde{\Delta}_t^{(\ell)}=\Delta_t^{r,(\ell)}-\gamma\Delta_t^{u,(\ell)}$
at every layer, regardless of the sign of $s_t^{(\ell)}$. This variant isolates
the contribution of the conditional layer-wise gating rule.

\noindent\textbf{(IV) Adapter Retention.}
Instead of discarding the temporary adapter after \emph{Rectify}, we keep it
frozen throughout \emph{Restore} and subsequent retained-data continued
training, and evaluate the model with the adapter attached. This variant
evaluates whether persistence depends on retaining the adapter beyond the
rectification stage.

\begin{table}[!t]
    \centering
    \caption{Component ablation results for the main CV federation with $K=50$
    clients. Rounds denote the combined \textnormal{\textit{Rectify}} and \textnormal{\textit{Restore}}
    communication rounds, excluding subsequent retained-data continued training;
    RR (\%) is measured at the final continued-training checkpoint. UF denotes
    unsuccessful forgetting.}
    \label{tab:component_ablation}
    \small
    \setlength{\tabcolsep}{4pt}
    \renewcommand{\arraystretch}{1.05}
    \resizebox{\columnwidth}{!}{%
    \begin{tabular}{llccccc}
        \toprule
        \textbf{Setting} & \textbf{Metric} & \textbf{(I)} & \textbf{(II)} &
        \textbf{(III)} & \textbf{(IV)} & \textbf{\method{}} \\
        \midrule
        \multirow{2}{*}{LeNet-5} & Rounds & $9$ & $23$ & $9$ & $24$ & $8$ \\
                                  & RR (\%) & $0.00$ & $82.5$ & $0.53$ & $0.47$ & $0.00$ \\
        \midrule
        \multirow{2}{*}{ResNet-18} & Rounds & $22$ & $10$ & $28$ & $47$ & $22$ \\
                                    & RR (\%) & $0.00$ & $94.7$ & $0.72$ & $0.65$ & $0.00$ \\
        \midrule
        \multirow{2}{*}{HSViT} & Rounds & $30$ & $16$ & $45$ & $39$ & $30$ \\
                                & RR (\%) & $0.05$ & UF & $0.90$ & $0.91$ & $0.03$ \\
        \midrule
        \multirow{2}{*}{Sample} & Rounds & $17$ & $2$ & $5$ & $6$ & $5$ \\
                                 & RR (\%) & $0.00$ & UF & $0.00$ & $0.00$ & $0.00$ \\
        \midrule
        \multirow{2}{*}{Class} & Rounds & $13$ & $18$ & $19$ & $18$ & $18$ \\
                                & RR (\%) & $0.00$ & UF & $0.00$ & $0.00$ & $0.00$ \\
        \bottomrule
    \end{tabular}
    }
        \vspace{0.5mm}
    \footnotesize{
        \textit{Note:}  UF denotes unsuccessful forgetting.
    }
\end{table}

\paragraph{Ablation results}
Table~\ref{tab:component_ablation} compares the combined \emph{Rectify} and
\emph{Restore} rounds with RR after subsequent retained-data continued
training. Without adapter-based amplification, variant (I) still achieves an RR
of $0.00\%$ in every evaluated setting. Its round count is close to that of the
complete method in the three client-level settings and is lower for class-level
unlearning, while sample-level unlearning requires $17$ rather than $5$ rounds.
These results show that the temporary adapter is optional and mainly improves
the efficiency of sample-level unlearning. Variant (II) retains \emph{Amplify}
but bypasses \emph{Rectify}; it yields RR values of $82.5\%$ and $94.7\%$ for
LeNet-5 and ResNet-18, respectively, and UF for the other three settings. This
identifies dual-stream rectification as the primary functional component.
Applying rectification without the layer-wise alignment condition in variant
(III) generally increases the required rounds or RR, supporting the conditional
layer-wise rule used by \method{}. Variant (IV) retains the frozen adapter
through \emph{Restore} and subsequent continued training but provides no
improvement, supporting its removal after \emph{Rectify} as specified in the
complete pipeline.

\subsection{CV and Sample-Level Sensitivity to \texorpdfstring{$\gamma$}{gamma}}
\label{sec:gamma-sensitivity}
Figure~\ref{fig:gamma_sweep} reports the communication rounds under different
values of $\gamma$ (lower is better). These values use the same round count as
Table~\ref{tab:component_ablation}: each value is the sum of the \emph{Rectify} and
\emph{Restore} rounds and excludes subsequent continued training. This
sensitivity experiment uses $K=10$ clients.
For client-level unlearning,
$\gamma=0.3$ gives the minimum for all three CV backbones: $7$ rounds for
LeNet-5, $8$ for ResNet-18, and $30$ for HSViT. HSViT reaches the same minimum
at $\gamma=0.4$, while $\gamma=0.3$ is shared by all three backbones. For
sample-level unlearning, $\gamma=1.5$ gives the minimum of $3$ rounds.

\begin{figure}[!ht]
    \centering
    \includegraphics[width=\columnwidth]{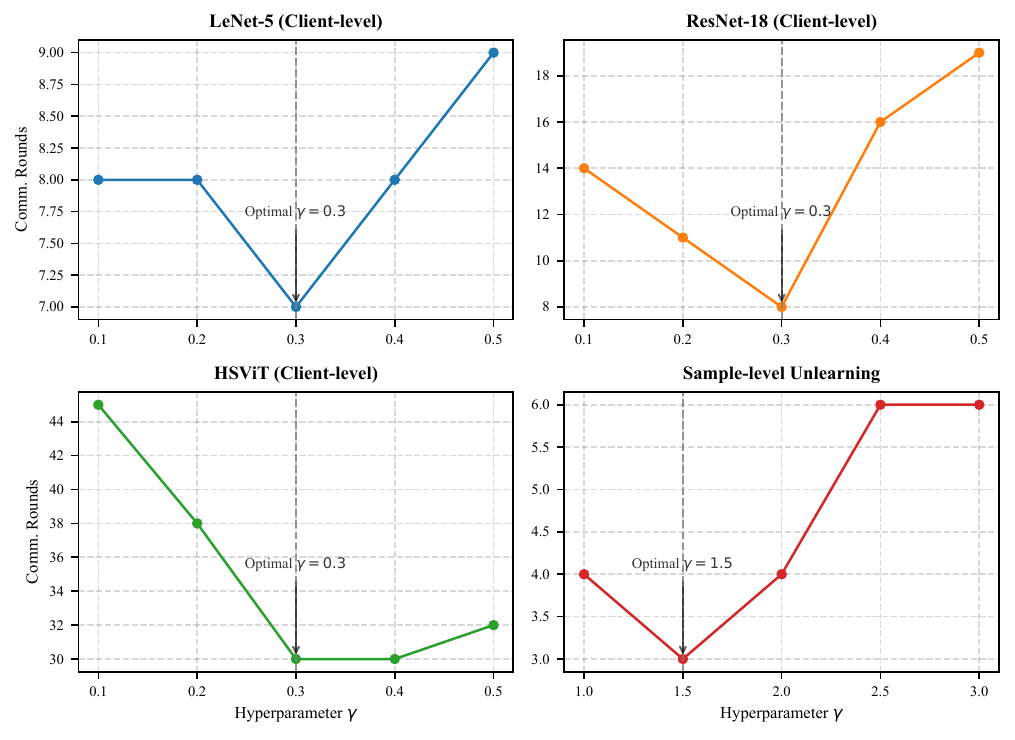}
    \caption{Sensitivity of the total communication rounds
    (\emph{Rectify} plus \emph{Restore}) to $\gamma$ (lower is better).
    The sensitivity experiments use $K=10$ clients.
}
    \label{fig:gamma_sweep}
\end{figure}

\subsection{NLP Sensitivity to \texorpdfstring{$\gamma$}{gamma}}
\label{sec:nlp-parameter-sensitivity}
Figure~\ref{fig:nlp_parameter_sensitivity} reports the model-specific sensitivity of
\textsc{Lethe} to the rectification strength $\gamma$. As in
Figure~\ref{fig:gamma_sweep}, the count includes both \emph{Rectify} and
\emph{Restore}. For
Llama-3.2-3B-Instruct, $\gamma=28$ achieves the target unlearning effect in
$7$ communication rounds. For OLMo-1B, $\gamma=0.9$ is the most efficient
configuration and requires $4$ rounds.

\begin{figure}[H]
    \centering
    \includegraphics[width=\columnwidth]{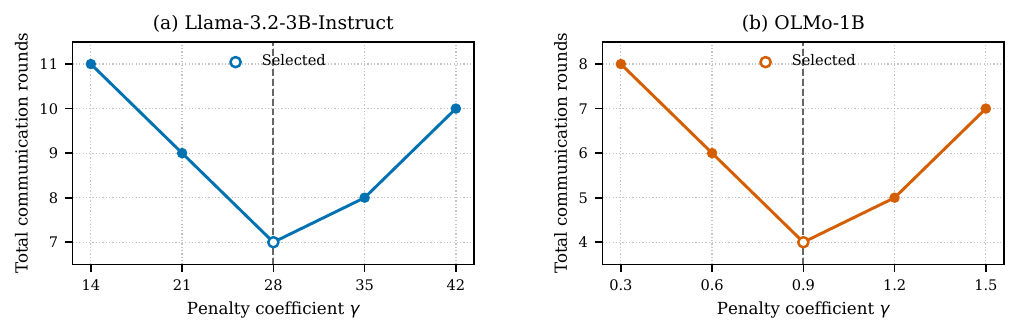}
    \caption{Model-specific NLP sensitivity to $\gamma$ using total
    communication rounds (\emph{Rectify} plus \emph{Restore}).}
    \label{fig:nlp_parameter_sensitivity}
\end{figure}

\clearpage
\endgroup

\end{document}